\documentclass{article}
\usepackage{graphicx} 
\usepackage{xcolor}
\usepackage{name}
\usepackage{subfigure}

\usepackage{bbm}
\usepackage{bm}
\usepackage{mathrsfs} 
\usepackage{amsmath}
\usepackage{amssymb}
\usepackage{amsthm}
\usepackage{amsfonts}
\usepackage{amsopn}

\usepackage{enumitem}
\usepackage{bbm}

\usepackage[utf8]{inputenc} 
\usepackage[T1]{fontenc}    
\usepackage{hyperref}       
\usepackage{url}            
\usepackage{booktabs}       
\usepackage{nicefrac}       
\usepackage{microtype}      
\usepackage{dsfont}

\global\long\def\E{\mathbb{E}}%
\global\long\def\tO{\widetilde{O}}%
\global\long\def\mS{\mathcal{S}}%
\global\long\def\mA{\mathcal{A}}%
\global\long\def\mX{\mathcal{X}}%
\global\long\def\R{\mathbb{R}}%

\global\long\def\tpi{\widetilde{\pi}}%

\newcommand{\var}{\operatorname{Var}}

\usepackage{xcolor}
\usepackage{todonotes} %
\usepackage[showdeletions]{color-edits} %

\title{Constant Stepsize Q-learning: Distributional Convergence, Bias and Extrapolation}

\author{Yixuan Zhang, Qiaomin Xie
    \texorpdfstring{\footnote{Emails: \texttt{yzhang2554@wisc.edu}, \texttt{qiaomin.xie@wisc.edu}}\\~\\
 \normalsize Department of Industrial and Systems Engineering\\ University of Wisconsin-Madison}{}
}

\date{\vspace*{-1cm}}
\begin{document}
\maketitle

\begin{abstract}
Stochastic Approximation (SA) is a widely used algorithmic approach
in various fields, including optimization and reinforcement learning
(RL). Among RL algorithms, Q-learning is particularly popular due
to its empirical success. In this paper, we study asynchronous Q-learning
with constant stepsize, which is commonly used in practice for its
fast convergence. By connecting the constant stepsize Q-learning 
to a time-homogeneous Markov chain, we show the distributional
convergence of the iterates in Wasserstein distance and establish 
its exponential convergence rate. We also establish a Central Limit Theory for Q-learning iterates, demonstrating the asymptotic normality of the averaged iterates. Moreover, we provide an explicit
expansion of the asymptotic bias of the averaged iterate in stepsize. 
Specifically, the bias is proportional to the stepsize
up to higher-order terms and we provide an explicit expression for the linear coefficient. This precise characterization of the bias
allows the application of Richardson-Romberg (RR) extrapolation technique
to construct a new estimate that is provably closer to the optimal
Q function. Numerical results corroborate our theoretical finding
on the improvement of the RR extrapolation method. 
\end{abstract}

\section{Introduction }

 Stochastic Approximation (SA) is a fundamental algorithmic paradigm
in various fields, including optimization, machine learning, stochastic
control and filtering, and reinforcement learning (RL). SA uses recursive
stochastic updates to solve fixed point equations. One prominent example
is the widely-used stochastic gradient descent (SGD) algorithm for finding the optimal solution to an objective
function \cite{lan2020first}. In the context of RL, well-known algorithms
such as Q-learning and TD-learning can be viewed as SA algorithms
for solving Bellman equations \cite{bertsekas1996neuro}. Classical
SA theory suggests using diminishing stepsize, ensuring asymptotic
convergence to the desired solution \cite{borkar08-SA-book}. 
However, SA with constant stepsize is commonly used in practice
due to its simplicity and faster convergence. In this case, SA iterates
can be viewed as a \emph{time-homogeneous Markov chain}. Adopting this perspective, a growing line of
recent work establishes weak convergence of constant stepsize SA and characterizes the stationary
distribution \cite{bianchi2022convergence,durmus2021riemannian,huo22-bias,Dieuleveut20-bach-SGD,Yu21-stan-SGD}.

In this paper, we investigate constant stepsize Q-learning, which
is an important instance of \emph{nonsmooth} SA with Markovian noise.
Q-learning is a popular RL algorithm that has played a
significant role in the empirical success of RL \cite{mnih2015human}.
It directly learns the optimal action-value function (or Q-function)
from sample trajectories. At each iteration,
Q-learning \emph{asynchronously} updates a single state-action pair based on one transition from the trajectory. Consequently, the
iterations inherently involve \emph{Markovian} noise resulting from the sampling
process of a Markov chain under the behavior policy. Q-learning's
asymptotic convergence and finite-sample guarantees have been extensively
studied \cite{Tsitsiklis1994-QLearn,borkar2000-ode-sa,Szepesvari97-QLearn-Rates,EvenDar04-QLearn-rates,chen2021lyapunov,li2020sample}.
These non-asymptotic results provide \emph{upper bounds }on either
the mean squared error (MSE) $\E[\Vert q_{k}-q^{*}\Vert _{\infty}^{2}\big]$
or high probability $\ell_{\infty}$ error $\Vert q_{k}-q^{*}\Vert _{\infty}.$

The main goal of this paper is to gain a more comprehensive understanding
of the behavior of constant stepsize Q-learning and its error decomposition.
In the discounted setting, Q-learning aims to solve the fixed point
equation involving the Bellman operator, which is contractive in the
\emph{nonsmooth} $\ell_{\infty}$ norm. Hence Q-learning can be viewed
as an instance of SA with a \emph{nonsmooth} operator and
Markovian noise. 
Recently, non-asymptotic analysis of Markovian SA has been
gaining attention \cite{Bhandari21-linear-td,srikant-ying19-finite-LSA,chen20-contract-SA,chen2021lyapunov,huo22-bias}.
However, these results either concern linear SA or provide {upper
bounds} on the error.

In this work, we study the Q-learning iterates through the lens of
Markov chain theory. We provide a more precise characterization of
the MSE in terms of the decomposition
\begin{equation}
\mathbb{E}\Vert q_{k}-q^{*}\Vert^{2}\asymp\underbrace{\big\Vert \mathbb{E}q_{k}-\mathbb{E}q_{\infty}^{(\alpha)}\big\Vert ^{2}}_{\text{optimization error}}+\underbrace{\big\Vert \mathbb{E}q_{\infty}^{(\alpha)}-q^{*}\big\Vert ^{2}}_{\text{asymptotic bias}}+\underbrace{\operatorname{Var}\left(q_{k}\right)}_{\text{variance}}\label{decomposition}
\end{equation}
where the random variable $q_{\infty}^{(\alpha)}$ denotes the limit of the Q-learning iterate $q_{k}$ with
stepsize $\alpha$. Our main results characterize the behavior of
the three terms above, as summarized in the following.
\begin{itemize}
\item By casting the constant stepsize Q-learning as a time-homogeneous
Markov chain, we establish the convergence of the iterates to a unique
random vector in $W_{2}$, the Wasserstein distance of order 2. Moreover,
$\{q_{k}\}_{k\geq0}$ converges to the limit $q_{\infty}^{(\alpha)}$
exponentially fast due to the use of a constant stepsize. This result
leads to an upper bound on the optimization error $\Vert \mathbb{E}q_{k}-\mathbb{E}q_{\infty}^{(\alpha)}\Vert $,
which also decays exponentially in $k$. 
The convergence rate cannot
be obtained from the existing upper bound on $\mathbb{E}\Vert q_{k}-q^{*}\Vert^{2}$
or $\left\Vert q_{k}-q^{*}\right\Vert _{\infty}$, which does not
vanish as $k\rightarrow\infty.$
\item We show that the variance $\operatorname{Var}\left(q_{k}\right)$
is of order $\mathcal{O}(1)$. By law of large numbers, one
can use the averaging procedure to eliminate the variance. Specifically,
the Polyak-Ruppert tail-averaged iterate has a variance of order
$\mathcal{O}(1/k)$. Consequently, for large $k$, the deviation between
the averaged iterate and the optimal solution $q^{*}$ for large $k$
is dominated by the asymptotic bias $\mathbb{E}q_{\infty}^{(\alpha)}-q^{*}$. 
\item We present an explicit expansion of the asymptotic bias $\mathbb{E}q_{\infty}^{(\alpha)}-q^{*}$
in the stepsize $\alpha$: 
\begin{equation}
\mathbb{E}q_{\infty}^{(\alpha)}-q^{*}=\alpha B+\widetilde{\mathcal{O}}(\alpha^{2}),\label{eq:bias_expansion}
\end{equation}
where $B$ is a vector independent of the stepsize $\alpha$ and $\widetilde{\mathcal{O}}$ denotes the variant of $\mathcal{O}$ that ignores logarithmic order.
Importantly, the bias characterization is exact, as equation (\ref{eq:bias_expansion})
is an \emph{equality }rather than an upper bound. Consequently, one
can use Richardson-Romberg extrapolation technique to reduce the bias
and obtain an estimate closer to $q^{*}$ with
order-wise smaller bias $\widetilde{\mathcal{O}}(\alpha^{2})$. 

\end{itemize}
\smallskip
On the technical side, to deal with the nonsmooth operator, we employ a \emph{local linearization} of the operator in the neighborhood of the optimal solution $q^{*}$. While local linearization has been explored in nonlinear SA literature, they mainly consider the asymptotic regime with diminishing stepsizes~\cite{lee2020unified,li2023statistical,melo2008analysis,gopalan2023demystifying}. We generalize this approach to characterize the dependence on the constant stepsize. It is worth noting that while the linear approximation component resembles similar behavior as linear SA \cite{huo22-bias}, a precise characterization of the bias requires a careful analysis of the linear approximation error to establish a proper higher order of $\alpha$.

\subsection{Related Work}
Here we discuss closely related work and defer other related work to Section~\ref{sec:appx_related} in supplementary materials. 

\paragraph{Q-learning.}  
An increasing volume of recent work has been dedicated
to understanding finite-time guarantees of Q-learning in various scenarios. 
There are two types of results
on the distance between the estimate $q_{k}$ and the optimal
Q-function $q^{*}$: high probability bounds and mean (square)
error bounds. 
For classical asynchronous Q-learning, as considered in this paper,
\cite{beck2012error} provide the first result on MSE
with constant stepsize and \cite{chen2021lyapunov} improve the result
by at least a $|\mathcal{S}||\mathcal{A}|$ factor. The work by \cite{li2023q} presents
the best known high probability sample complexity. 
It is worth noting that these two types of bounds are not directly
comparable, as discussed in 
\cite{chen2021lyapunov}.
Importantly, these results
are achieved either by rescaled linear stepsize $\alpha_{k}=a/(b+k)$
\cite{qu2020finite,chen2021lyapunov} or by a carefully
chosen constant stepsize based on the target accuracy \cite{chen2021lyapunov,li2023q}.
Contrasting with these findings, our results provide a precise characterization
of the convergence rate as well as the bias induced by constant-stepsize
$\alpha$, for any $\alpha$ in a given range. Our explicit characterization
enables the application of RR technique, leading to an estimate
with reduced bias, while simultaneously enjoying the exponential
convergence of the optimization error. 

Some recent work also studies Polyak-Ruppert averaged Q-learning. \cite{xie2022statistical} and \cite{li2023statistical} prove a functional central limit theorem for the averaged iterates of \emph{synchronous} Q-learning with constant stepsize and diminishing stepsize, respectively. In contrast, we focus on asynchronous Q-learning involving Markovian data.

\paragraph{Stochastic approximation.} 
There is a growing interest in investigating general SA
with constant stepsize. Most work along this line considers i.i.d.\ or martingale difference noise, and establishes
finite-time guarantees for contractive/linear SA~\cite{chen20-contract-SA,Mou20-LSA-iid,durmus2021-LSA} or SGD~\cite{Dieuleveut20-bach-SGD,Yu21-stan-SGD}. 
Recent work investigates constant-stepsize SA with Markovian noise, motivated by applications in RL. For linear SA, the work by \cite{srikant-ying19-finite-LSA} provides finite-time upper bounds on the MSE. \cite{Mou21-optimal-linearSA} study LSA with PR averaging and presents instance-dependent MSE upper bounds with tight dimension dependence. The work by \cite{durmus2021stability}
shows a finite-time upper bound for the $p$-th of LSA iterate on general state space. 
The paper \cite{lauand2022bias} shows that LSA with Markovian noise admits a biass that can not be eliminated by averaging.
The work \cite{huo22-bias} establishes the distributional convergence
of LSA iterates, and provides an explicit asymptotic expansion of the bias in stepsize. Going beyond LSA, the work \cite{chen2022finite}
considers contractive SA under a strong monotone condition and provides
finite-time upper bound on the MSE.

Our results have some similarities to \cite[Proposition 2]{Dieuleveut20-bach-SGD},
\cite[Theorem 3]{durmus2021-LSA} and \cite{huo22-bias}, in that we
also study instances of SA with constant stepsizes through Markov chain theory. However, our setting is 
different from \cite[Theorem 3]{durmus2021-LSA} as the sampling process
in RL naturally induces Markovian noise, whereas they consider i.i.d.\ data.
While the work \cite{huo22-bias} also considers Markovian noise, their
focus is on linear SA. In contrast, Q-learning 
involves nonsmooth update, which brings additional challenges on the analysis of convergence and bias. In particular, for convergence proof, the difference between two coupled LSA iterates can be reformulated as an LSA; however, this is not the case for Q-learning, which requires a novel analysis for the coupled iterates. 
For the bias analysis,  we employ
a local linearization method to decompose the Q-learning operator into a linear term and a remaining approximation term. While the technique for LSA \cite{huo22-bias} can be used to analyse the linear part, it is highly nontrivial to show the remaining term is of higher order dependence on $\alpha$. We establish this result by analyzing the fourth moment of the iterates. Our techniques may be of independent interest and have the potential to be applied to the analysis of other nonsmooth/nonlinear SA algorithms. 

\section{Preliminaries}

Consider a discounted Markov decision
process (MDP) defined by the tuple $(\mS,\mA, T, r,\gamma),$ where $\mS$ and $\mA$ are the (finite)
state space and action space,  $T:\mS\times \mA \rightarrow \Delta(\mS)$ is the transition kernel,
$r:\mS\times \mA \rightarrow [0,r_{\max}]$ is the reward function, and $\gamma\in(0,1)$ is the discounted factor. At time $t \in \{0,1,\dots\}$, 
the system is in state $s_t \in \mS$;  upon taking action $a_t\in\mA$, the system transits to $s_{t+1} \in \mS$ 
with probability $T(s_{t+1} | s_t,a_t)$  and generates a reward $r_t=r(s_t,a_t).$  

A stationary policy $\pi:\mS\rightarrow \Delta(\mA)$ maps each state to a distribution over the actions, where $\pi(a | s)$ represents the
probability of taking action $ a$ given state $s$.  For each policy $\pi$,  the action-value function (Q-function) is defined as follows: $\forall s \in \mS, \forall a\in \mA, $ 
$q^{\pi}(s, a)  =\E\big[ \sum_{t=0}^{\infty}\gamma^{t}r(s_{t},a_{t}) | s_0=s, a_0=a \big], \mbox{where $a_k\sim \pi(\cdot | s_k) $ for all } k\geq 0.$
An optimal policy $\pi^{*}$ is the policy that maximizes $q^{\pi}(s,a)$ for all $s\in\mS$ and $a\in \mA$ simultaneously. 
It is well known that such an optimal policy exists~~\cite{bertsekas2017dynamic} and we denote the associated Q functions as $q^* \equiv q^{\pi^{*}}$. Notably, given $q^*,$ one can obtain the optimal policy $\pi^*(s)\in \arg\max_{a\in \mA} Q^*(s,a)$.

\textbf{Behavior policy.} The goal of reinforcement learning (RL) is to learn the optimal policy based on transition data from the system, without the knowledge of the MDP model $(T,r).$ In this paper, we consider off-policy setting, where we have access to a sample trajectory $\{s_k,a_k,r_k\}_{k\geq 0}$ generated by the MDP under a fixed \emph{behavior policy} $\tpi.$ 

Define $\mX:=\mS\times \mA\times \mS,$ and let $x_k = (s_k, a_k, s_{k+1}).$ Note that when $\tpi$ is stationary, $\{x_k\}_{k\geq 0}$ forms a time-homogeneous Markov chain. We use $P=(p_{ij})$ to denote the corresponding transition matrix. In this paper, we impose the following assumption on the behavior policy.

\begin{assumption}\label{MC}
$\{x_k\}_{k \geq 0}$ is an irreducible and aperiodic Markov chain on a finite state $\mathcal{X}$ with stationary distribution $\mu_{\mathcal{X}}$. Also, the distribution of the initial state $x_0$ is $\mu_{\mathcal{X}}.$
\end{assumption}

Assumption \ref{MC} is equivalent to assuming that Markov chain $\{s_k,a_k\}_{k\geq0}$ induced by the behavior policy $\tpi$ is uniformly ergodic with a unique stationary distribution $\mu_S$~\cite{chen2021lyapunov}. This assumption is standard for analyzing off-policy Q-learning~\cite{li2020sample,chen2021lyapunov,qu2020finite}.  
Assumption \ref{MC} implies that  $\{x_k\}_{k \geq 0}$ mixes geometrically fast to the  stationary distribution $\mu_{\mathcal{X}}$ \cite{levin2017markov}, and there exist $c \geq 0$ and $\rho \in (0,1)$ s.t.
\begin{equation}
\max \limits_{x \in \mathcal{X}} \Vert p^k(x,\cdot) - \mu_{\mathcal{X}}(\cdot)\Vert_{TV} \leq c \rho^k,
\end{equation}
where $p^k(x,\cdot)$ denotes the distribution of $x_k$ given $x_0 = x.$

To quantify how fast $\{x_k\}_{k \geq 0}$ mixes to a specified precision, we define the mixing time below.

\begin{definition} \label{def:mixing_time}
$\forall \delta > 0$, define $t_{\delta} := \min \{ k \geq 0: \max \nolimits_{x \in \mathcal{X}} \Vert p^k(x,\cdot) - \mu_{\mathcal{X}}(\cdot)\Vert_{TV} \leq \delta\}.$
\end{definition}
Under Assumption \ref{MC}, we have $t_{\alpha} \leq \frac{\log(c/\rho) + \log(1/\alpha)}{\log(1/\rho)}$, which implies $\lim_{\alpha \to 0} \alpha^{m_1} t_{\alpha^{m_2}} = 0$,  $\forall m_1, m_2 > 0$. 
We  assume that $x_0 \sim \mu_{\mathcal{X}}$ to simplify some presentation. This assumption is not essential and can be relaxed by adapting our result after the Markov chain $\{x_k\}_{k \geq 0}$ has almost mixed. We remark that the same assumption is considered in many previous works \cite{Bhandari21-linear-td,huo22-bias,Mou21-optimal-linearSA}.

\textbf{Q-learning.} The Q-learning algorithm \cite{Watkins92-QLearning} is an iterative method for estimating the function $Q^*$ based on the sample trajectory $\{s_k,a_k,r_k\}_{k\geq 0}$. It generates a sequence of Q-function estimate $\{q_k:\mS\times\mA \rightarrow \R\}_{k},$  according to the following recursion:
\begin{align}\label{eq:tabular}
    q_{k+1} &= q_k + \alpha_k F(x_k,q_k),
\end{align}

where $\alpha_k$ is the stepsize. Here the operator $F: \mathcal{X} \times\mathbb{R}^{|\mathcal{S}\Vert\mathcal{A}|}   \mapsto \mathbb{R}^{|\mathcal{S}\Vert\mathcal{A}|}$, known as empirical Bellman operator, is defined as: $\forall (s,a)\neq (s_k,a_k), [F( x, q)](s, a)=0;$ 
\[   [F( x, q)](s_k, a_k)
   =r(s_k, a_k)+ \gamma \max \limits_{v} q_k(s_{k+1}, v) - q_k(s_k, a_k).\]
In this paper, we focus on constant stepsize $\alpha_k \equiv \alpha>0$. We use superscript $q_k^{(\alpha)}$ to emphasize the dependence on the stepsize $\alpha,$ but omit it when it is clear from the context. 

We record a few basic properties of Q-learning. (1) By the boundedness of reward, there exists a constant $q_{\max}$ such that $\Vert q_k\Vert_\infty \leq q_{\max},\forall k.$ (2) Denote the expected operator of $F$ by $\bar{F}(q):=\E_{x\sim \mu_{\mX}}[F(x,q)]$. It has been shown that $\Bar{F}(q) + q$ is a $\beta$-contraction mapping w.r.t.\ $\Vert\cdot\Vert_\infty$~\cite{chen2021lyapunov},  where 
\begin{align}
    \beta &= 1 - (1-\gamma)\min\nolimits_{(s,a)} \mu_{\mathcal{S}}(s,a). \label{eq:beta}
\end{align}
Recall that $\mu_{\mathcal{S}}$ is the stationary distribution of Markov chain $\{s_k,a_k\}_{k \geq 0}$~\cite{chen2021lyapunov}. By Assumption \ref{MC}, we have $\min_{(s,a)} \mu_{\mathcal{S}}(s,a) >0 $, thus $\beta < 1$. (3) Crucially, the iterates $\{q_k\}$ generated by Q-learning is not a Markov chain. On the other hand, we can see that the joint process $\{x_k,q_k\}_{k\geq 0}$ is a Markov chain on the state space $\mX \times \R^{|\mS|\times|\mA|}.$ In particular, employing a constant stepsize $\alpha$ induces a time-homogeneous Markov chain $\{x_k,q_k\}_{k\geq 0}.$

Part of our results on Q-learning (cf.~Theorem~\ref{thm:bias}) requires the following assumption. 
\begin{assumption}\label{assum:smooth}
    The optimal policy $\pi$ is unique. That is,  $\exists \Delta > 0$ such that for $ \forall s \in \mathcal{S}$,
$
q^*(s,a^*_s) - q^*(s,a) \geq 2\Delta, \forall a \neq a^*_s,
$
where $a^*_s := \arg\max \limits_{a} q^*(s,a)$ denotes the optimal action for each state $s.$
\end{assumption}

Similar conditions have been considered in prior work on the analysis of Q-learning variants~\cite{devraj2017fastest,li2023statistical}. Assumption~\ref{assum:smooth} implies that the operator in \eqref{eq:tabular} can be approximated by local linearization around $q^*$ and high-order approximation error, which leads to our precise characterization of the bias induced by constant stepsize. 

\textbf{Additional notations.} Part of our analysis uses the reversed Markov chains. By Assumption~\ref{MC}, the Markov chain $\{x_k\}_{k\geq 0}$ is irreducible. An implication is that the chain $\{x_k\}_{k\geq 0}$ running backward in time is also a Markov chain \cite{norris1998markov}, with transition kernel $\hat{P} = (\hat{p}_{ij})$ given by $\mu_{\mX}(j) \hat{p}_{j i}=\mu_{\mX}(i) p_{i j}.$

\section{Main Results}

In this section, we present our main results. 
In Section \ref{sec:convergence}, we show that joint data Q-learning iterates $\{x_k,q_k\}_{k\geq 0}$, viewed a time-homogeneous Markov chain, converges to a unique limit distribution. We also establish explicit convergence rates. Moreover, we show a central limit theorem (CLT) for the iterates $\{q_k\}_{k \geq 0}$ in Section \ref{sec:CLT}. We then precisely characterize the relationship between the limit and the stepsize in
Section \ref{sec:bias}. Furthermore,  we investigate the implication of these results for Polyak-Ruppert tail averaging and Richardson-Romberg extrapolation in Section \ref{sec:average}.

\subsection{Stationary Distribution and Convergence Rate} \label{sec:convergence}

Note that the Q-learning iterate $\{q_k\}_{k\geq 0}$ is not a Markov chain by itself, as its dynamic depends on the Markovian data $\{x_k\}_{k \geq 0}$. To show the distributional convergence of $\{q_k\}_{k\geq 0}$ and quantify its convergence rate, we consider the joint process $\{x_k,q_k\}_{k\geq 0},$ which can be cast as a time-homogeneous Markov chain. We will analyze the convergence of this Markov chain using the Wasserstein 2-distance. The Wasserstein 2-distance between two distributions $\mu$ and $\nu$ in $\mathcal{P}_2(\mathbb{R}^d)$ is defined as
\begin{align*}
W_2(\mu, \nu) & =\inf _{\xi \in \Pi(\mu, \nu)}\left(\int_{\mathbb{R}^d}\|u-v\|_\infty^2 \mathrm{~d} \xi(u, v)\right)^{\frac{1}{2}} \\
&=\inf \left\{\left(\mathbb{E}\left[\left\|\theta-\theta^{\prime}\right\|_\infty^2\right]\right)^{\frac{1}{2}}: \mathcal{L}(\theta)=\mu, \mathcal{L}\left(\theta^{\prime}\right)=\nu\right\},
\end{align*}
where $\mathcal{P}_2(\mathbb{R}^d)$ denotes the space of square-integrable distributions on $\mathbb{R}^d$, $\mathcal{L}(\theta)$ denote the distribution of $\theta$ and $\Pi(\mu, \nu)$ is the set of all joint distributions in $\mathcal{P}_2(\mathbb{R}^d \times \mathbb{R}^d)$ with marginal distributions $\mu$ and $\nu$. To analyze the Markov chain $\{x_k,q_k\}_{k \geq 0}$, we define the extended Wasserstein 2-distance as in \cite{huo22-bias}. 
Let $d_0\left(x, x^{\prime}\right):=\mathds{1}\left\{x \neq x^{\prime}\right\}$ and $\bar{d}\left((x, \theta),\left(x^{\prime}, \theta^{\prime}\right)\right):=\sqrt{d_0\left(x, x^{\prime}\right)+\left\|\theta-\theta^{\prime}\right\|_\infty^2}$, which defines a metric on $\mathcal{X} \times \mathbb{R}^d$. Then, the extended Wasserstein 2-distance for two distributions $\Bar{\mu}$ and $\Bar{\nu}$ in $\mathcal{P}_2(\mathcal{X} \times \mathbb{R}^d)$ with respect to the metric $\Bar{d}$ is as below:
\begin{equation}\label{extendw2}
\begin{aligned}
\bar{W}_2(\bar{\mu}, \bar{\nu}) & =\inf \left\{\left(\mathbb{E}\left[\bar{d}\left(z, z^{\prime}\right)^2\right]\right)^{1 / 2}: \mathcal{L}(z)=\bar{\mu}, \mathcal{L}\left(z^{\prime}\right)=\bar{\nu}\right\} .
\end{aligned}
\end{equation}

We show that the Markov chain $\{x_k,q_k\}_{k \geq 0}$ converges in $\Bar{W}_2$ to a unique stationary distribution, \emph{geometrically} fast, as stated in the following Theorem.

\begin{thm}\label{limit4tabular}
Suppose that Assumption \ref{MC} holds, and the stepsize $\alpha$ for Q-learning \eqref{eq:tabular} satisfies
\begin{equation}
    \alpha t_\alpha \leq c_0 \frac{(1-\beta)^2}{\log(|\mathcal{S}\Vert\mathcal{A}|)}
    \qquad \text{for some constant } c_0.
    \label{eq:alpha_condition_convergence}
\end{equation} 
\begin{enumerate}[leftmargin=*]
    \item Under all initial distribution of $q_0$, the sequence
    $\{x_k,q_k\}_{k \geq 0}$ converges in $\Bar{W_2}$ to a 
    a unique limit $(x_\infty, q_\infty) \sim \Bar{\mu}$. Moreover, we have 
    \begin{equation}
        \operatorname{Var}(q_\infty) \leq c_{Q} \frac{\log (|\mathcal{S}\Vert\mathcal{A}|)}{\left(1-\beta\right)^2} \alpha t_\alpha, 
    \end{equation}
    where $c_{Q} = 912e\left(3\Vert q^*\Vert_\infty + r_{\max}\right).$
    \item $\Bar{\mu}$ is the unique stationary distribution of the Markov chain $\{x_k,q_k\}_{k \geq 0}$.
    \item Let $\mu := \mathcal{L}(q_\infty)$ be the second marginal of $\Bar{\mu}$. For all $k \geq t_\alpha$, we have 
\begin{equation}\label{convergerate}
\begin{aligned}
W_2^2\left(\mathcal{L}(q_k), \mu\right)
&\leq 24 \eta^{k-t_\alpha}\left(\mathbb{E}\left[\Vert q_0\Vert_\infty^2\right] + \mathbb{E}\left[\Vert q_\infty\Vert_\infty^2\right]\right),  
\end{aligned}    
\end{equation}
where $\eta=1-{\left(1-\beta\right) \alpha}/{2}.$
\end{enumerate}
\end{thm}


Theorem \ref{limit4tabular} states that the Markov chain $\{x_k,q_k\}_{k \geq 0}$ admits a unique stationary distribution, given that the constant stepsize $\alpha$ satisfies condition \eqref{eq:alpha_condition_convergence}.  Recall that under Assumption~\ref{MC}, for all $ m_1, m_2 > 0$, we have $\lim_{\alpha \to 0} \alpha^{m_1} t_{\alpha^{m_2}} = 0$. Therefore, there always exists a small enough stepsize $\alpha$ such that condition \eqref{eq:alpha_condition_convergence} holds. We remark that the resulting limit distribution $\Bar{\mu}$ is not generally a product distribution of its marginals $\mu_{\mX}$ and $\mu.$ 

Note that the convergence results stated in Theorem \ref{limit4tabular} cannot be obtained from existing error bounds on Q-learning. For example, the sharpest high probability bound on $\ell_{\infty}$ error scales as $\|q_k-q^*\|_{\infty} \lesssim (1-\rho)^k\|q_0-q^*\|_{\infty}+\mathcal{O}(\sqrt{\alpha})$, where $\rho\in (0,1)$ \cite{li2020sample}. Another type of upper bound is on the MSE that scales as $\E[\|q_k-q^*\|_{\infty}^2 ]\lesssim (1-(1-\beta)\alpha/2)^{k-t_\alpha}\|q_0-q^*\|^2_{\infty}+\mathcal{O}(\alpha t_{\alpha})$ \cite{chen2021lyapunov}. 
Both upper bounds imply that the sequence eventually falls in a neighbor of the optimal solution $q^*$ and the initial condition is forgotten exponentially fast. However, these result does not imply the distributional convergence of the sequence $\{q_k\}_{k\geq 0}$ or its convergence rate in $W_2$ metric.

We would like to highlight the techniques employed to prove Theorem~\ref{limit4tabular}. A standard method for proving the positive recurrence of a Markov chain is to verify irreducibility and Lyapunov drift condition~\cite{Meyn12_book}, as used in prior work on SA~\cite{borkar2021ode} and SGD~\cite{Yu21-stan-SGD}. However, this method requires strong condition on the randomness of the Markov chain dynamics, which is difficult to verify in Q-learning. Instead, we draw inspiration from recent work on LSA~\cite{huo22-bias}, and prove weak convergence by showing the convergence in $W_2$ distance through coupling arguments. Wasserstein distance has recently been used \cite{Dieuleveut20-bach-SGD,huo22-bias,durmus2021stability}. We remark that the coupling argument in our proof is more involved due to the nonsmooth nature of the update operator $F$. 
We sketch the proof outline of Theorem~\ref{limit4tabular} in Section~\ref{sec:proof_outline_convergence}. The complete proof is deferred to Section~\ref{sec:proof_limit4tabluar}.

A direct consequence of the convergence in $W_2$ metric is the convergence of the first two moments. We can also obtain explicit convergence rates from Theorem~\ref{limit4tabular}, as detailed in the following corollary. The proof is provided in Section~\ref{sec:proof_co4limit}.

\begin{cor}\label{co4limit}
Under the setting of Theorem \ref{limit4tabular}, for all $k \geq t_\alpha$,
\begin{equation*}
\begin{aligned}
\Vert\mathbb{E}[q_k - q_\infty]\Vert_{\infty}^2 
&\leq C\cdot \big(1-{\left(1-\beta\right) \alpha}/2\big)^{k-t_\alpha},\\
\left\|\mathbb{E}\left[q_k q_k^{\top}\right]-\mathbb{E}\left[q_{\infty} q_{\infty}^{\top}\right]\right\|_\infty &\leq C^{\prime}\cdot \big(1-{\left(1-\beta\right) \alpha}/{2}\big)^{\frac{k-t_\alpha}{2}}.
\end{aligned}
\end{equation*}
where $C$ and $C^{\prime}$ are constants independent of $\alpha$ and $k$. 
\end{cor}

\subsection{Central Limit Theorem} \label{sec:CLT}
Building on the convergence result, we establish a CLT for $\{q_k\}_{k \geq 0}$. Here we define $S_n = \sum_{k = 0}^{n-1}\big(q_k - \mathbb{E}[q_\infty]\big)$ and $Y_n(t) = n^{-\frac{1}{2}}S_{\lfloor nt \rfloor}.$ Let $\mathcal{D} = \mathcal{D}[0,1]$ denote the Skorokhod space, which is a separable and complete function space under some proper metrics \cite{prokhorov1956convergence}.

\begin{thm}\label{thm:clt}
Under the setting of Theorem \ref{limit4tabular}, $\boldsymbol{\Sigma}:=\lim _{n \rightarrow \infty} \frac{1}{n} \mathbb{E}_\pi\left(S_n S_n^{\top}\right)$ exists, and for $\bar{\mu}$-almost every point $(x_0,q_0)$ the sequence $\{n^{-\frac{1}{2}}S_n\}_{n\geq 0}$ converge in distribution to the $|\mathcal{S}||\mathcal{A}|$-dimensional Gaussian distribution $\mathcal{N}\left(\mathbf{0}, \boldsymbol{\Sigma}\right)$. Furthermore, the process $(Y_n(t))_{0 \leq t \leq 1}$ converges weakly to $\left(\boldsymbol{\Sigma}^{\frac{1}{2}} \boldsymbol{B}(t)\right)_{0 \leq t \leq 1}$ on the Skorokhod space $D[0,1]$, where $\boldsymbol{B} = (\boldsymbol{B}(t))_{t \geq 0}$ is the standard Brownian motion.
\end{thm}

Theorem \ref{thm:clt} states that the average of Q-learning iterates is asymptotically normally distributed around the expected value of the unique stationary distribution. Establishing such a CLT is important for uncertainty quantification and statistical inference \cite{li2023statistical}. Similar result has been established for synchronous Q-learning with constant stepsize \cite{xie2022statistical}. It is worth highlighting that one key step in \cite{xie2022statistical} uses the Kantorovich–Rubinstein theorem \cite{edwards2011kantorovich} defined on a Wasserstein distance with single-step contraction. However, such result does not hold in our setting due to Markovian data. To this end, we use the result in Theorem~\ref{limit4tabular} and ergodicity of $\{x_k\}_{k\geq 0}$ to establish CLT. The detailed proof is provided in Section~\ref{sec:proof_clt}.

\subsection{Bias Expansion} \label{sec:bias}

Under constant stepsize $\alpha$, Theorem \ref{limit4tabular} asserts that the  convergence of $q^{(\alpha)}_k$ to $q^{(\alpha)}_\infty$, which is of distribution $\mu$. Therefore, the  estimates $q^{(\alpha)}_k$ of Q-learning with constant stepsize do not converge to a point, but oscillate around the mean $\E[q^{(\alpha)}_{\infty}]$. 
Here we would like to quantify the \emph{bias}, i.e., the deviation of the mean $\E[q^{(\alpha)}_{\infty}]$ from the optimal $q^*.$
One of our main contributions is to provide an \emph{explicit} asymptotic expansion in the step-size $\alpha$ of the bias $\E[q^{(\alpha)}_{\infty}]-q^*.$

\begin{thm}\label{thm:bias}
Suppose that Assumptions \ref{MC} and \ref{assum:smooth} hold and $\alpha \leq \alpha_0$ for some $\alpha_0.$ Then the following holds for a vector $B$ that is independent of $\alpha.$
\begin{align} \label{eq:bias}
\mathbb{E}\left[q_\infty\right] &= q^* + \alpha B + \mathcal{O}({\alpha^2 + \alpha^2t_{\alpha^2}^2}).
\end{align}
where $B = B(r, \gamma, P)$ is explicitly given in the appendix.

\end{thm}

A few remarks are in order. 

The above theorem states that the asymptotic bias of Q-learning can be decomposed into a linear term and a higher order term of $\alpha$ under Assumption \ref{assum:smooth}. We emphasize that our bias characterization of the linear dependence on $\alpha$ is \emph{exact}. As discussed in the previous sub-section, existing results are typically in the form of an upper bound on the bias. Specifically, the high probability upper bound on $\ell_{\infty}$ error \cite{li2020sample} implies an bias of $\mathcal{O}(\sqrt{\alpha}).$ In contrast, our analysis reveals a refined result with  $ \alpha B+\widetilde{\mathcal{O}}(\alpha^{2}) $ bias. 

One key step in the proof of Theorem \ref{thm:bias} is to calculate $\mathbb{E}[F(x_\infty,q_\infty) \mid x_\infty = i], \forall i \in \mathcal{X}$. For linear SA, this step is straightforward. However, for asynchronous Q-learning, operator $F$ is not linear, and not even smooth, which makes the analysis more complicated. In our proof, we provide a local linearization method which can bridge the gap between nonlinear SA and LSA.
 We outline the proof of Theorem \ref{thm:bias} in Section~\ref{sec:proof_outline_bias}. The complete proof is provided in Section~\ref{sec:proof_bias}. %

We remark that the coefficient $B$ of the linear term is independent of $\alpha$. It only depends on the underlying MDP and the behavior policy. One can find an explicit expression of $B$ in the proof (cf.~Equation \eqref{eq:B}). Importantly, for the special case where the associated data sequence $\{x_k\}_{k\geq 0}$ is i.i.d., we have $B=\mathbf{0}$. However, the bias term $\mathcal{O}({\alpha^2 + \alpha^2t_{\alpha^2}^2})$ still remains, due to nonlinearity of Q-learning operator. This should be contrasted with the LSA where the bias vanishes with i.i.d.\ data \cite{huo22-bias}. In general, the existence of bias implies that the mean of the sequence $\{q_k\}_{k \geq 0}$ limit deviates from the optimal solution $q^*.$ 
Therefore,  averaging the iterates $q_k$ does not eliminate the bias. However, thanks to the independence of $B$ on $\alpha,$ we can leverage an extrapolation technique to reduce the bias.

\subsection{Tail Average and Richardson-Romberg  extrapolation}\label{sec:average}

We now utilize the bias expansion result~Theorem \ref{thm:bias} to study the behavior of Q-learning when combined with Polyak-Ruppert average and Richardson-Romberg extrapolation.

\subsubsection{Polyak-Ruppert Averaging} 

The celebrated Polyak-Ruppert averaging procedure \cite{Ruppert88-Avg,Polyak92-Avg} can reduce the variance of the estimate and accelerate the convergence rate. In this paper, we consider the tail-averaging variant of PR-averaging \cite{Jain18-tail-avg}, which is defined as follows with a burn-in period $k_0$:
\begin{align} 
\bar{q}_{k_0, k}:=&\frac{1}{k-k_0} \sum\nolimits_{t=k_0}^{k-1} q_t, \quad\mbox{for } k\geq k_0+1.\label{eq:tail_PR}
\end{align}

The following corollary provides non-asymptotic results for the first and second moments of $\bar{q}_{k_0, k}$. The proof is provide in Section~\ref{sec:proof_coro_moment}.

\begin{cor}\label{co:pr}
Under the setting of Theorem~\ref{thm:bias}, the tail-averaged iterates \eqref{eq:tail_PR} satisfy the following:  $\forall k > k_0 \geq t_{\alpha^2}$: 
\begin{align} 
\mathbb{E}\left[\bar{q}_{k_0, k}\right]-q^*=& \alpha B+ \mathcal{O}({\alpha^2 + \alpha^2t_{\alpha^2}^2})+\mathcal{O}\left(\frac{1}{\alpha(k - k_0)} \exp\left(-\frac{\alpha(1-\beta)k_0}{4}\right)\right),\label{eq:PR_first} \\
\mathbb{E}\left[\left(\bar{q}_{k_0, k}-q^*\right)\left(\bar{q}_{k_0, k}-q^*\right)^{\top}\right] =& \underbrace{\alpha^2 B^{\prime} + \mathcal{O}({\alpha^3 + \alpha^3t_{\alpha^2}^2})}_{\substack{\textup { asymptotic squared bias }}}+\underbrace{\mathcal{O}\left(\frac{1}{(k-k_0)\alpha}\right)}_{\textup{variance}} \nonumber\\
&+ \underbrace{\mathcal{O}\left(\frac{1}{(k - k_0)^2\alpha^2}\exp\left(-\frac{\alpha(1-\beta)k_0}{4}\right)\right)}_{\textup {optimization error }},\label{eq:PR_second}
\end{align}
where $B$ and $B'$ are independent of $\alpha.$
\end{cor}

For simplicity, let us consider the case $k_0=k/2$ and discuss the mean squared distance between the averaged-iterate $\bar{q}_{k/2, k}$ and $q^*$. From the analysis, we note the MSE can be decomposed into three parts: the asymptotic squared bias term $\| \E[\bar{q}_{\infty/2, \infty}-q^*]\|^2$ that is independent of $k$ and averaging; the second part for the variance of $\bar{q}_{k/2, k}$ that scales as $1/k$; and the optimization error $\| \E[\bar{q}_{\infty/2, \infty}-\bar{q}_{k/2, k}]\|^2$ that decays to 0 geometrically fast.
Importantly,  the larger the stepsize $\alpha$ is, the faster the variance and optimization error decay. This finding also justifies the benefit of using constant-stepsize.

\subsubsection{Richardson-Romberg Extrapolation}

Given the explicit expansion of the bias in stepsize $\alpha$ (cf.~Theorem~\ref{thm:bias}), we can leverage the Richardson-Romberg (RR) extrapolation technique from numerical analysis~\cite{gautschi2011numerical} to reduce the bias. Specifically, consider running two Q-learning recursions using the \emph{same} data stream $\{x_k\}_{k \geq 0}$, but with different stepsizes $\alpha$ and $2\alpha.$ Denote by $\bar{q}^{(\alpha)}_{k_0, k}$ and $\bar{q}^{(2\alpha)}_{k_0, k}$ the corresponding tail-averaged iterates. The corresponding RR extrapolated iterates are given by
\begin{align} \label{eq:RR_def}
\widetilde{q}^{(\alpha)}_{k_0, k} = 2\bar{q}^{(\alpha)}_{k_0, k} - \bar{q}^{(2\alpha)}_{k_0, k}.
\end{align}
With $k_0,k \to \infty$, Theorems \ref{limit4tabular} and \ref{thm:bias} imply that $\widetilde{q}^{(\alpha)}_{k_0, k}$ converges to $2q_\infty^{(\alpha)} - q_\infty^{(2\alpha)}$, which has a bias
\begin{align*}
2\mathbb{E} q_{\infty}^{(\alpha)} -\mathbb{E} q_{\infty}^{(2 \alpha)} - q^* 
&=2\left(\alpha B+\mathcal{O}({\alpha^2 + \alpha^2t_{\alpha^2}^2})\right)-\left(2 \alpha B+\mathcal{O}({\alpha^2 + \alpha^2t_{\alpha^2}^2})\right) \\
&=\mathcal{O}({\alpha^2 + \alpha^2t_{\alpha^2}^2}).
\end{align*}
Note that compared with $q^{(\alpha)}_{\infty}$ and $q^{(2\alpha)}_{\infty},$ the extrapolated sequence provides a new estimate that reduces the bias by a factor of $\alpha.$ We formally state the result in the following corollary, which quantifies the non-asymptotic behavior of the first two moments of extrapolated sequence $\{\widetilde{q}^{(\alpha)}_{k_0, k}\}_{k\geq0}.$ The proof is provided in Section~\ref{sec:proof_rr}.

\begin{cor}\label{rr}
Under the setting of Theorem \ref{thm:bias}, the RR extrapolated iterates \eqref{eq:RR_def} with stepsizes $\alpha$ and $2\alpha$ satisfy the following for all $k > k_0 \geq t_{\alpha^2}$: 
\begin{align}
&\mathbb{E}\left[\widetilde{q}_{k_0, k}^{(\alpha)}\right]-q^*=  \mathcal{O}({\alpha^2 + \alpha^2t_{\alpha^2}^2})+\mathcal{O}\left(\frac{1}{\alpha(k - k_0)} \exp\left(-\frac{\alpha(1-\beta)k_0}{4}\right)\right), \label{eq:RR_first} \\
&   \mathbb{E}\big[\big(\widetilde{q}_{k_0, k}^{(\alpha)}-q^*\big)\big(\widetilde{q}_{k_0, k}^{(\alpha)}-q^*\big)^{\top}\big] =\mathcal{O}\left({\alpha^4 + \alpha^{4}t_{\alpha^2}^4}\right)  
   + {\mathcal{O}\left(\frac{1}{(k-k_0)\alpha}\right)}+ {\mathcal{O}\left(    \frac{1}{(k - k_0)^2\alpha^2}\exp\left(-\frac{\alpha(1-\beta)k_0}{4}\right)  \right)}. \label{eq:RR_second}
\end{align}
\end{cor}

Let us compare the MSE bounds~\eqref{eq:PR_second} on the PR-averaged iterates and
extrapolated sequence~\eqref{eq:RR_second}. It is important to note that the asymptotic squared bias is reduced from $\mathcal{O}(\alpha^2)$ to roughly $\mathcal{O}(\alpha^4)$ by RR extrapolation! Meanwhile, RR extrapolation still enjoys similar decaying rates of variance and optimization error. We remark that the RR procedure involves the computation of two parallel Q-learning iterates, which can use either the same or different data sequences. This makes the RR procedure inherently parallelizable, offering potential performance improvements when implemented on parallel computing architectures.

\section{Proof Outline}

\subsection{Proof Outline for Theorem~\ref{limit4tabular} on Convergence} \label{sec:proof_outline_convergence}

This sub-section is devoted to the discussion of the proof outline for the existence of the limit distribution, which is the most challenging part.

We first note that the space $\mathcal{P}(\mX\times \R^{|\mS|\times|\mA|})$ endowed with our extended 
 Wasserstein 2-distance $\Bar{W}_2$ is a Polish space, as indicated by Theorem 6.18 in \cite{villani2009optimal}. If we can establish that $\sum_{k = 0}^{\infty} \Bar{W}_2^2\left(\mathcal{L}\left(x_k,q_k\right), \mathcal{L}\left(x_{k+1}, q_{k+1}\right)\right)< \infty$, then the sequence $\{x_k, q_k\}_{k \geq 0}$ forms a Cauchy sequence. This, in turn, allows us to prove the existence of the limit distribution, using the fact that all Cauchy sequences converge in a Polish space.

The next step involves coupling through the construction of two Markov chains,$\{x_k^{[1]}, q_k^{[1]}\}_{k \geq 0}$ and $\{x_k^{[2]}, q_k^{[2]}\}_{k \geq 0},$ which share the same underlying data stream $\{x_k^{[1]}\}_{k \geq 0}=\{x_k^{[2]}\}_{k \geq 0}=\{x_k\}_{k \geq 0}$. We observe that the iterates difference $w_k := q_{k}^{[1]} - q_{k}^{[2]}$ exhibits the following dynamic, leading to the subsequent Proposition:
$
w_{k+1}(s_k, a_k) = (1-\alpha)w_k(s_k, a_k) + \alpha\gamma \big( \max \limits_{a} q_{k}^{[1]}(s_{k+1}, a) - \max \limits_{a} q_{k}^{[2]}(s_{k+1}, a) \big).
$

\begin{prop}\label{w}
Under Assumption \ref{MC}, if $\alpha t_\alpha \leq c_0 \frac{(1-\beta)^2}{\log(|\mathcal{S}\Vert\mathcal{A}|)}$, the following bound holds for all $k \geq t_\alpha$,
\[\mathbb{E}\left[\left\|w_k\right\|_{\infty}^2\right] \leq 12\mathbb{E}\left[\Vert w_0\Vert_\infty^2\right] \left(1-\frac{\left(1-\beta\right) \alpha}{2}\right)^{k-t_\alpha}.\]
\end{prop}

The key idea behind the proof of Proposition \ref{w} is to exploit the fact that the difference between two max operators can be lower bounded by the minimum of the difference and upper bounded by the maximum of the difference. This inpsires us to construct two new dynamics that serve as lower and upper bounds on $\{w_k\}_{k \geq 0}$. We prove that the lower/upper bound sequences decay geometrically fast to 0, which immediately implies a geometric decay rate of $\{w_k\}_{k \geq 0}$
 for all initial distribution of $q_0^{[1]}$ and $q_0^{[2]}$. Then, by carefully choosing the initial distribution of $q_0^{[2]}$, we can  ensure that $(x_k, q^{[2]}_k) \overset{d}{=} (x_{k+1}, q^{[1]}_{k+1})$. Consequently, $\bar{W}_2^2\big(\mathcal{L}\big(x_k,q_k\big), \mathcal{L}\big(x_{k+1}, q_{k+1}\big)\big)\rightarrow 0$ geometrically fast. This result can be applied to show $\sum_{k = 0}^{\infty} \bar{W}_2^2\big(\mathcal{L}\big(x_k,q_k\big), \mathcal{L}\big(x_{k+1}, q_{k+1}\big)\big)< \infty$, which establishes the existence of limit distribution.

\subsection{Proof Outline for Theorem~\ref{thm:bias} on Bias Expansion} \label{sec:proof_outline_bias}

In the proof of Theorem \ref{thm:bias}, a crucial technique employed is the linearization of the non-smooth operator $F(x,q)$. Specifically, for a fixed $x$, we linearize $F(x,q)$ around the optimal solution $q^*$ as per the following proposition.

\begin{prop}\label{gradient}
There exists a function $F_{q^*}^{\prime}: \mathcal{X}  \mapsto  \mathbb{R}^{|\mS\Vert\mA| \times |\mS\Vert\mA|}$ s.~t.~for any $(x, q) \in \mathcal{X} \times \mathbb{R}^{|\mS\Vert\mA|}$
\begin{equation}\label{eq:gradient}
F(x,q) = F(x,q^*) + (G_{q^*}(x) - I_d)(q - q^*) + R(x,q),
\end{equation}
with $d = |\mathcal{S}\Vert\mathcal{A}|$ and $\|R(x,q)\|_\infty = \mathcal{O}\left(\|q - q^*\|_\infty^4\right).$ Furthermore, $\mathbb{E}_{x \sim \mu_{\mathcal{X}}}G_{q^*}(x)$ doesn't have an eigenvalue of 1.

\end{prop}
We next provide a finite-time upper bound on the fourth moment of the error, which shows the remaining term $R(x,q)$ is of a higher order of $\alpha$ by Proposition \ref{gradient}. We remark that  existing non-asymptotic results for Q-learning are limited to the first moment and second moment of the error.

\begin{prop}\label{44tabular} 
Suppose that Assumption~\ref{MC} holds. Consider Q-learning \eqref{eq:tabular} with constant stepsize $\alpha.$ There exists a constant $\alpha_0>0$ such that $\forall \alpha \in (0, \alpha_0)$ and $\forall k\geq t_{\alpha^2},$ 
\begin{equation}\label{eq:4th_bound}
\mathbb{E}[\Vert q_k - q^*\Vert_\infty^4] \leq b_1(1 -\alpha (1 - \gamma)^{2})^{k - t_{\alpha^2}} +b_2\alpha^{2} + b_3\alpha^2 t_{\alpha^2}^2,
\end{equation}
where $b_1$, $b_2$ and $b_3$ are  constants  independent of $\alpha$.
\end{prop}

Note that the first term on RHS of equation \eqref{eq:4th_bound} decays geometrically in $k$, whereas the remaining two terms are independent of $k$. Consequently, as $k \to \infty$, the upper bound is of order $\mathcal{O}\left(\alpha^{2}+\alpha^2t_{\alpha^2}^2\right)$.

The proof is inspired by the work \cite{chen2021lyapunov} that constructs a Generalized Moreau Envelope (GME) $M(\cdot)$ to analyse $\|\cdot\|_\infty^2$. By some nice property of GME~\cite{chen2020finite, chen2021lyapunov}, we can derive the bound for $M(\cdot)^2$, which equivalently provides the bound for $\|\cdot\|_\infty^4.$ We defer the complete proof of Proposition~\ref{gradient} and \ref{44tabular} to Section~\ref{sec:linearization}.

Therefore, the RHS of equation \eqref{eq:gradient} can be viewed as a combination of a linear operator and a high-order remaining term $R(x,q)$. Assume we are in the limit $(x_\infty, q_\infty)$. Recall that Theorem \ref{44tabular} has established a non-asymptotic bound for the fourth moment. Then, by fatou's lemma, 
$$\mathbb{E}[\|q_\infty - q^*\|_\infty^4] \leq \lim \inf_{k \to \infty} \mathbb{E}[\|q_k - q^*\|_\infty^4] = \mathcal{O}\left(\alpha^{2}+\alpha^2t_{\alpha^2}^2\right).$$
We then can analyze the dynamic of $\{x_k,q_k\}_{k \geq 0}$ as a linear SA combined with a remaining term of order $\mathcal{O}\left(\alpha^{2}+\alpha^2t_{\alpha^2}^2\right).$

\section{Numerical Experiments}

We consider two MDPs: the first example is a $1 \times 3$ Gridword with two actions (left/right); the second one is a classical $4 \times 4$ Gridworld with the slippery mechanism in Frozen-Lake, and four actions (left/up/right/down).
For both MDPs, the discounted factor is $\gamma = 0.9$ and the Markovian data $\{x_k\}_{k \geq 0}$ is generated from a uniformly random behavior policy. 
We defer details of reward function and transition kernel for the MDPs to Section~\ref{sec:appx_experiment}.

For Markovian data case, we run Q-learning with initialization $q_0^{(\alpha)} = q^* + 10$ and stepsize $\alpha \in \{0.1,0.2,0.4\}$. We also consider two
diminishing stepsizes: a rescaled linear stepsize $\alpha_k = 1/\big(1+(1-\gamma)k\big)$ as suggested by prior work~\cite{qu2020finite,chen2020finite} and a polynomial stepsize $\alpha_k = 1/{k^{0.75}}$.
The simulation results for the two MDPs are illustrated in Figure~\ref{fig:Markov_simple_MDP} and \ref{fig:Markov_hard_MDP}. We plot the $\ell_1$-norm error $\|\bar{q}_{k/2,k}^{(\alpha)} - q^*\|_1$ for the tail-averaged (TA) iterates $\bar{q}_{k/2,k}^{(\alpha)}$,  the RR extrapolated iterates $\widetilde{q}_{k}^{(\alpha)}$ with stepsizes $\alpha$ and $2\alpha$, and iterates with diminishing stepsizes. 

We first observe that the larger the stepsize $\alpha,$ the faster it converges, as implied by Corollary~\ref{co:pr}. 
We note that the final TA error, which corresponds to the asymptotic bias, is approximately proportional to the stepsize, as indicated by the roughly equal space between three 
TA lines in the log scale plots. Moreover, RR extrapolated iterates reduce the bias, which can be observed by comparing, e.g, the solid orange line (TA with $\alpha = 0.2$) and the dotted red line (RR with $\alpha = 0.2$ and $0.4$). These results are consistent with Corollary~\ref{rr}.
Furthermore, the TA and RR-extrapolated iterates with constant stepsizes enjoy significantly faster initial convergence than those with diminishing stepsizes. Particularly for the more complicated the MDP, as shown in Figure \ref{fig:Markov_hard_MDP}, iterates with diminishing stepsize converge slowly, while TA and RR-extrapolated iterates converge quickly and then saturate. 
A general choice of diminishing stepsize is of the form $\alpha_k=a/(b+k^c),$ where $a,b$ and $c$ are hyper-parameters. 
Tuning the best hyper-parameters for diminishing stepsize is generally more challenging than a single parameter for  constant stepsize.

\begin{figure}[htbp]
  \centering
    \subfigure[$1\times 3$ Gridworld.]{              
        \includegraphics[width=0.45\textwidth]{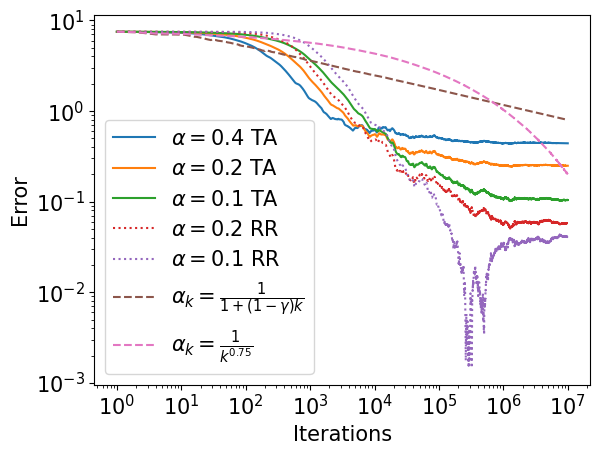}\label{fig:Markov_simple_MDP}}
    \hspace{15pt}
    \subfigure[$4\times 4$ Gridworld.]{
        \includegraphics[width=0.45\textwidth]{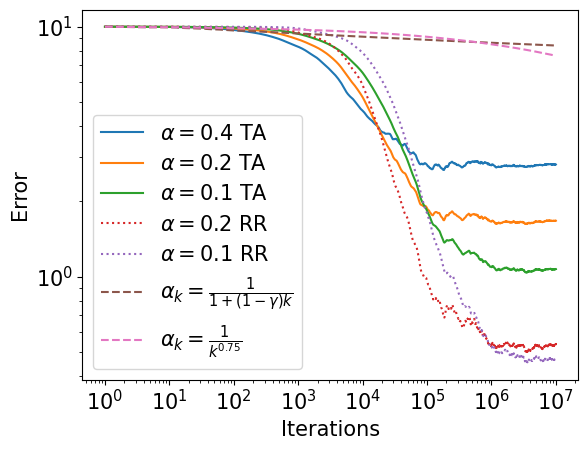}\label{fig:Markov_hard_MDP}}
    \caption{The errors of tail-averaged (TA) iterates and RR extrapolated iterates with different stepsizes.}\label{fig:markovian} 
\end{figure}

We also perform a similar set of experiments for MDPs with linear function approximation. We observe similar behaviors of the TA iterates and RR extrapolated iterates as the tabular case. Due to space constraint, we defer the details to Section~\ref{sec:appx_experiment}.

\section{Conclusions}

In this work, we provide a more comprehensive study of asynchronous Q-learning with constant stepsizes, through the framework of Markov chain theory. We establish the distributional convergence of the iterates, characterize the convergence rate, and prove a central limit theorem for the averaged iterates. Our convergence results lead to a refined characterization of the error. In particular, the explicit expansion of the asymptotic bias w.r.t.\ stepsize $\alpha$ allows one to use the RR extrapolation for bias reduction. There are a few interesting directions one can take to extend our work. First, our CLT, together with our bias characterization and the Richardson-Romberg de-biasing scheme, allow one to create confidence intervals for the output of the Q-learning algorithms. 
Second, we conjecture that our results extend to linear function approximation, as demonstrated by our empirical results. It will be interesting to generalize our analysis for this case. Our current results requires the assumption on the local linearity in the neighborhood of the optimal solution. Extending our analysis without this assumption is a direction worth pursuing.

\bibliography{arxiv.bib}
\bibliographystyle{plain}

\appendix
\section{More Related Work} \label{sec:appx_related}

\paragraph{Q-learning} Earlier work established the asymptotic convergence of Q-learning algorithm with diminishing stepsize \cite{Tsitsiklis1994-QLearn,Szepesvari97-QLearn-Rates}. Over the past few years, an increasing volume of work has been dedicated
to understanding finite-time guarantees of Q-learning in various scenarios. 
from tabular setting \cite{beck2012error,chen2021lyapunov,qu2020finite,wainwright2019stochastic,li2023q} to function approximation \cite{chen2022finite,xu2020finite,du2020agnostic,cai2019neural}. In this paper we focus on the classical asynchrounous Q-learning. There is another variant of Q-learning that concerns an \emph{synchronous} setting, where all state-action pairs are updated simultaneously at each step. This setting requires access to a simulator, which generates independent samples for each state-action pair. For synchronous Q-learning, the best-known sample complexity for mean error bound is $\tO({SA}{(1-\gamma)^{-5}\epsilon^{-2}})$ \cite{wainwright2019stochastic,chen2020finite}. The paper \cite{li2021tightening}
provides the state-of-art high probability sample complexity $\tO(\frac{SA}{(1-\gamma)^{4}\epsilon^{2}})$. In this paper, we focus on the classical asynchronous Q-learning which updates only a single state-action pair upon each observation. The Markovian noise inherited in the asynchronous model makes it considerably more challenging to analyze than the synchronous case.

We also note that there are other lines of work focusing on Q-learning
variants that aim to accelerate convergence and improve sample complexity,
such as variance-reduced Q-learning \cite{li2020sample,wainwright2019variance,sidford2018optimal},
 speedy Q-learning \cite{azar2011speedy} and double Q-learning \cite{weng2020mean}.
 Another direction considers Q-learning with sophisticated exploration
 strategies, with an emphasis on regret bound \cite{jin2018q,bai2019provably}.
 Regret is a metric fundamentally different from finite-sample bounds,
 and techniques for these two types of guarantees are quite different.
 A comparison with these results is beyond the scope of this paper.

\textbf{Stochastic approximation}. 
There is a rich literature on the study of SA. 
Classical SA theory mainly focuses on the asymptotic
convergence \cite{kushner2003-yin-sa-book,borkar08-SA-book,borkar2000-ode-sa,Blum54-SA},
typically assuming a diminishing stepsize sequence. More recent studies
have shifted the focus to non-asymptotic results. In particular, there
is a growing interest in investigating general SA and SGD algorithms
with constant stepsize. Most work along this line considers SA or
SGD with i.i.d.\ or martingale difference noise, and establishes
finite-time bounds. The paper \cite{chen20-contract-SA} considers
contractive SA and presents an upper bounds on the MSE. \cite{Lakshminarayanan18-LSA-Constant-iid}
analyzes linear SA (LSA) and establishes finite-time upper and lower
bounds on the MSE. The work \cite{Mou20-LSA-iid} refines these results,
providing tight bounds with the optimal dependence on problem-specific
constants as well as a central limit theorem (CLT) for the averaged
iterates. There are also some recent studies developing new bounds
on random matrix products to analyze LSA: \cite{durmus2021-LSA} establishes
tight concentration bounds of LSA, and \cite{durmus22-LSA} extends
these bounds to LSA with iterate averaging. In the context of SGD,
the work in \cite{Dieuleveut20-bach-SGD} considers strongly convex
and smooth functions. They prove that the iterates converge to a unique
stationary distribution by Markov chain theory. Subsequent work generalizes
this result to non-convex and non-smooth functions with quadratic
growth \cite{Yu21-stan-SGD}, and proves asymptotic normality of the
averaged SGD iterates. The work \cite{chen21-siva_asymptotic} exams
the limit of the stationary distribution as stepsize goes to zero.
All these results are established under the i.i.d.\ noise setting.
Additionally, \cite{bianchi2022convergence} explores SGD for non-smooth
non-convex functions with martingale difference noise, and establishes
the weak convergence of the iterates to the set of critical points
of the objective function.

\section{Proof of Theorem \ref{limit4tabular}} \label{sec:proof_limit4tabluar}

In this section, we provide the proof of Theorem \ref{limit4tabular}. The first part of the proof, Section~\ref{sec:proof_limit4tabular_coupling}, involves coupling through the construction of two iterates of Q-learning. Using the result of this step, we then establish the existence and uniqueness of the stationary distribution for the joint Markov chain $(x_k,q_k)_{k \geq 0}$ (part 1 and 2 of Theorem~\ref{limit4tabular}) in Section~\ref{sec:proof_limit4tabular_stationary}. We prove the convergence rate (part 3 of Theorem~\ref{limit4tabular}) in Section \ref{sec:proof_limit4tabular_rate}.

\subsection{Coupling and Geometric Convergence} \label{sec:proof_limit4tabular_coupling}

We construct a pair
of coupled Markov chains, $(x_k, q^{\left[1\right]}_k)_{k \geq 0}$ and $(x_k, q^{\left[2\right]}_k)_{k \geq 0}$, defined as 
\begin{equation}\label{couple}
\begin{aligned}
q_{k+1}^{[1]}(s_k, a_k) & = q_{k}^{[1]}(s_k, a_k)+\alpha\left(r(s_k, a_k) + \gamma \max \limits_{a} q_{k}^{[1]}(s_{k+1}, a) - q_{k}^{[1]}(s_k, a_k)\right), \\
q_{k+1}^{[2]}(s_k, a_k) & = q_{k}^{[2]}(s_k, a_k)+\alpha\left(r(s_k, a_k) + \gamma \max \limits_{a} q_{k}^{[2]}(s_{k+1}, a) - q_{k}^{[2]}(s_k, a_k)\right).
\end{aligned}
\end{equation}
Here $(q^{\left[1\right]}_k)_{k \geq 0}$ and $(q^{\left[2\right]}_k)_{k \geq 0}$ are two iterates generated by the Q-learning algorithm, coupled by sharing the
underlying data stream $(x_k)_{k \geq 0}$. We assume that the initial iterates $q^{\left[1\right]}_0$
and $q^{\left[2\right]}_0$ may depend on each other and
on $x_0$, but are independent of $(x_k)_{k \geq 1}$ given $x_0$.

Define the iterates difference as $w_k := q_{k}^{[1]} - q_{k}^{[2]}$. Note that the dynamic for $\{w_k\}_{k \geq 0}$ can be formulated as follows:
\[
w_{k+1}(s_k, a_k) = (1-\alpha)w_k(s_k, a_k) + \alpha\gamma \left( \max \limits_{a} q_{k}^{[1]}(s_{k+1}, a) - \max \limits_{a} q_{k}^{[2]}(s_{k+1}, a) \right).
\]

We can exploit the dynamic of $\{w_k\}_{k \geq 0}$ to establish its convergence rate, as stated in Proposition~\ref{w}. The proof of Proposition~\ref{w} is deferred to Section~\ref{sec:proof_w}. 

When $\alpha t_\alpha \leq c_0 \frac{(1-\beta)^2}{\log(|\mathcal{S}\Vert\mathcal{A}|)}$, we can apply Proposition~\ref{w} to bound the square of $W_2$ distance between $q_{k}^{[1]}$ and $q_{k}^{[2]}$ as follows: for all $k \geq t_\alpha$,
\begin{equation}\label{geo}
\begin{aligned}
W_2^2\left(\mathcal{L}\left(q_k^{[1]}\right), \mathcal{L}\left(q_k^{[2]}\right)\right) & \stackrel{(\mathrm{i})}{\leq} \bar{W}_2^2\left(\mathcal{L}\left(x_k, q^{[1]}_k\right), \mathcal{L}\left(x_k,q^{[2]}_k\right)\right) \\
& \stackrel{(\mathrm{ii})}{\leq} \mathbb{E}\left[\left\|q_k^{[1]}-q_k^{[2]}\right\|_\infty^2\right] \\
& = \mathbb{E}\left[\left\|w_k\right\|_{\infty}^2\right] \\
&\stackrel{(\mathrm{iii})}{\leq} 12\mathbb{E}\left[\Vert w_0\Vert_\infty^2\right] \left(1-\frac{\left(1-\beta\right) \alpha}{2}\right)^{k-t_\alpha},
\end{aligned}    
\end{equation}
where the inequality $(\mathrm{i})$ follows from the definition of $W_2$ and $\Bar{W}_2$; the inequality $(\mathrm{ii})$
holds as the $\Bar{W_2}$ is defined by an infimum as in equation \eqref{extendw2}; the inequality $(\mathrm{iii})$ follows from applying Proposition \ref{w}.

Therefore, $W_2^2\left(\mathcal{L}\left(q_k^{[1]}\right), \mathcal{L}\left(q_k^{[2]}\right)\right)$ decays geometrically. We will use this result in the next sub-section to prove that $(x_k,q_k)_{k \geq 0}$ converges to a unique stationary distribution.

\subsection{Existence and Uniqueness of Stationary Distribution} \label{sec:proof_limit4tabular_stationary}

Note that equation \eqref{geo} always holds for any joint distribution of initial iterates $(x_0, q^{[1]}_0, q^{[2]}_0)$. After fixing an arbitrarily chosen distribution of $(x_0, q^{[1]}_0)$, we need to carefully choose the conditional distribution of $q^{[2]}_0$ to  ensure that $(x_k, q^{[2]}_k) \overset{d}{=} (x_{k+1}, q^{[1]}_{k+1})$ holds for all $k \geq 0$, where $\overset{d}{=}$ denotes equality in distribution. Recall that $\hat{P}$ represents the transition kernel for the time-reversed Markov chain of $(x_k)_{k \geq 0}$,
and  the initial distribution of $x_0$ is assumed to be mixed already. Given a specific $x_0$, we sample $x_{-1}$  from $\hat{P}(\cdot \mid x_0)$. Additionally, we use $q^{[2]}_{-1}$ to denote a random variable that satisfies $q^{[2]}_{-1} \overset{d}{=} q^{[1]}_0$ and is independent of $(x_k)_{k \geq 0}$. Finally, we set $q^{[2]}_0$ as 
\begin{equation}\label{watchback}
q^{[2]}_0 = q^{[2]}_{-1} + \alpha F(x_{-1},q^{[2]}_{-1}).    
\end{equation}

By the property of time-reversed Markov chains, we have 
 $(x_k)_{k \geq -1} \overset{d}{=} (x_k)_{k \geq 0}.$
Given that $q^{[2]}_{-1} \overset{d}{=} q^{[1]}_0$ and $q^{[2]}_{-1}$ is independent with $(x_k)_{k \geq -1}$, we can prove  $(x_k, q^{[2]}_k) \overset{d}{=} (x_{k+1}, q^{[1]}_{k+1})$ for all $k \geq 0$ by comparing the
dynamic of $(q^{[1]}_k)_{k \geq 0}$ and $(q^{[2]}_k)_{k \geq 0}$ as given in equations \eqref{couple} and \eqref{watchback}.

We thus have for all $k \geq t_\alpha$:
$$
\begin{aligned}
\bar{W}_2^2\left(\mathcal{L}\left(x_k, q^{[1]}_k\right), \mathcal{L}\left(x_{k+1}, q^{[1]}_{k+1} \right)\right) & =\bar{W}_2^2\left(\mathcal{L}\left(x_k, q^{[1]}_k\right), \mathcal{L}\left(x_k, q^{[2]}_k\right)\right) \\
& \leq 12\mathbb{E}\left[\Vert w_0\Vert_\infty^2\right] \left(1-\frac{\left(1-\beta\right) \alpha}{2}\right)^{k-t_\alpha},
\end{aligned}
$$
where the second inequality follows from equation \eqref{geo}. It follows that
$$
\begin{aligned}
&\sum_{k = 0}^{\infty} \bar{W}_2^2\left(\mathcal{L}\left(x_k,q^{[1]}_k\right), \mathcal{L}\left(x_{k+1}, q^{[1]}_{k+1}\right)\right)\\
\leq & \sum_{k = 0}^{t_\alpha - 1} \bar{W}_2^2\left(\mathcal{L}\left(x_k,q^{[1]}_k\right), \mathcal{L}\left(x_{k+1}, q^{[1]}_{k+1} \right)\right) + 12\mathbb{E}\left[\Vert w_0\Vert_\infty^2\right] \sum_{k = 0}^{\infty} \left(1-\frac{\left(1-\beta\right) \alpha}{2}\right)^{k}\\
< & \infty,
\end{aligned}
$$
where the last step holds since $\frac{\left(1-\beta\right) \alpha}{2} \in (0,1)$. Consequently, $(\mathcal{L}(x_k, q^{[1]}_k ))_{k \geq 0}$ forms a Cauchy sequence with respect to the metric $\Bar{W}_2$. Since the space $\mathcal{P}_2(\mathcal{X} \times \mathbb{R}^d  )$ endowed with $\Bar{W}_2$ is a Polish space, every Cauchy sequence converges \cite[Theorem 6.18]{villani2009optimal}. Furthermore, convergence
in Wasserstein 2-distance also implies weak convergence  \cite[Theorem 6.9]{villani2009optimal}. Therefore, we conclude that the sequence $(\mathcal{L}(x_k, q^{[1]}_k))_{k \geq 0}$ converges weakly to a limit distribution $\Bar{\mu} \in \mathcal{P}_2(\mathcal{X} \times \mathbb{R}^d  )$.

Next, we show that $\Bar{\mu}$ is independent of the initial iterate distribution of $q_0^{[1]}$, when $x_0$ is initialized from its unique stationary distribution $\mu_{\mX}$. Suppose there exists another sequence $(x_k, \widetilde{q}_k^{[1]})_{k \geq 0}$ with a different initial distribution  that converges to a limit $\widetilde{\mu}$. By triangle inequality, we have
$$
\bar{W}_2(\bar{\mu}, \tilde{\mu}) \leq \bar{W}_2\left(\bar{\mu}, \mathcal{L}\left( x_k,q_k^{[1]}\right)\right)+\bar{W}_2\left(\mathcal{L}\left(x_k,q_k^{[1]}\right), \mathcal{L}\left(x_k,\widetilde{q}_k^{[1]}\right)\right)+\bar{W}_2\left(\mathcal{L}\left(x_k,\widetilde{q}_k^{[1]}\right), \tilde{\mu}\right) \stackrel{k \rightarrow \infty}{\longrightarrow} 0.
$$
Note that the last step holds since $\bar{W}_2\left(\mathcal{L}\left(x_k, q_k^{[1]}\right), \mathcal{L}\left(x_k, \widetilde{q}_k^{[1]}\right)\right) \stackrel{k \rightarrow \infty}{\longrightarrow} 0$ by equation \eqref{geo}. We thus have $\Bar{W}_2(\Bar{\mu}, \widetilde{\mu}) = 0,$ which implies the uniqueness of the limit $\Bar{\mu}$.

Moreover, we will show that the unique limit distribution $\mu$ is also a stationary distribution for the Markov chain $(x_k,q_k)_{k \geq 0}$, as stated in the following lemma.

\begin{lem}\label{invariance}
Let $(x_k, q_k)_{k \geq 0}$ and $(x_k^{\prime} , q_k^{\prime})_{k \geq 0}$ be two trajectories of Q-learning iterates, where $\mathcal{L}\left(x_0, q_0\right) = \Bar{\mu}$ and $\mathcal{L}(x_0^{\prime} , q_0^{\prime}) \in  \mathcal{P}_2(\mathcal{X} \times \mathbb{R}^d  )$ is arbitrary. Under Assumption \ref{MC} we have
\[\Bar{W}^2_2\left(\mathcal{L}\left( x_1, q_1\right), \mathcal{L}(x_1^{\prime}, q_1^{\prime})\right) \leq \rho\Bar{W}^2_2\left(\mathcal{L}\left(x_0, q_0\right), \mathcal{L}(x_0^{\prime},q_0^{\prime})\right),\]
where the quantity $\rho: = \max\left(1+ 2(\alpha R_{\max} + \alpha\gamma q_{\max})^2, 2(1 + \alpha\gamma)^2\right)$ is independent of   $\mathcal{L}(x_0^{\prime}, q_0^{\prime})$. In particular, for any $k \geq 0$, if we set $\mathcal{L}(x_0^{\prime}, q_0^{\prime}) = \mathcal{L}(x_k, q_k)$, then
\[\Bar{W}^2_2\left(\mathcal{L}\left(x_1, q_1\right), \mathcal{L}(x_{k+1}, q_{k+1})\right) \leq \rho\Bar{W}^2_2\left(\Bar{\mu}, \mathcal{L}(x_k , q_k )\right).\]
\end{lem}

\begin{proof}[Proof of Lemma \ref{invariance}]
We prove this lemma by coupling the two processes $(x_k, q_k)_{k \geq 0}$ and $(x_k^{\prime}, q_k^{\prime} )_{k \geq 0}$ such that 
$$
\begin{aligned}
\Bar{W}^2_2\left(\mathcal{L}\left(x_0,q_0\right), \mathcal{L}(x_0^{\prime}, q_0^{\prime})\right) & = \mathbb{E}\left[d_0(x_0 , x_0^{\prime}) + \Vert q_0 - q^{\prime}_0\Vert_\infty^2\right] \text{ and}\\
x_{k+1} & = x^{\prime}_{k+1} ~~\text{ if } x_k = x^{\prime}_{k},~~~\forall k  \geq 0.
\end{aligned}$$

Since $\Bar{W}_2$ is defined by infimum over all couplings, we have 
\[\Bar{W}^2_2\left(\mathcal{L}\left(x_1, q_1\right), \mathcal{L}(x_1^{\prime}, q_1^{\prime})\right) \leq \mathbb{E}\left[d_0(x_1 , x_1^{\prime}) + \Vert q_1 - q^{\prime}_1\Vert_\infty^2\right].\]
 Recall the definition of the discrete metric $d_0(x_0^{\prime}, x_0):= \mathds{1}\{x_0^{\prime} \neq x_0\}$. We denote by $e_{(s,a)} \in \mathbb{R^{|\mathcal{S}\Vert\mathcal{A}|}}$ the one-hot vector with only one “1” in the location of $(s,a)$. We then have
$$
\begin{aligned}
\Vert q_1 - q^{\prime}_1 \Vert_\infty & = \Vert q_0 - q^{\prime}_0 - \alpha e_{(s_0,a_0)}q_0(s_0,a_0) +\alpha e_{(s^{\prime}_0,a^{\prime}_0)}q^{\prime}_0(s^{\prime}_0,a^{\prime}_0)\\
&\quad + \alpha e_{(s_0,a_0)}r(s_0,a_0) -\alpha e_{(s^{\prime}_0,a^{\prime}_0)}r(s^{\prime}_0,a^{\prime}_0)\\
&\quad + \alpha \gamma e_{(s_0,a_0)}\max \limits_{a} q_0(s_1,a) -\alpha  \gamma e_{(s^{\prime}_0,a^{\prime}_0)}\max \limits_{a}q^{\prime}_0(s^{\prime}_1,a)\Vert_\infty\\
& \leq \Vert q_0 - q^{\prime}_0 - \alpha e_{(s_0,a_0)}q_0(s_0,a_0) +\alpha e_{(s^{\prime}_0,a^{\prime}_0)}q^{\prime}_0(s^{\prime}_0,a^{\prime}_0)\Vert_\infty\\
&\quad+ \alpha\Vert e_{(s_0,a_0)}r(s_0,a_0) -e_{(s^{\prime}_0,a^{\prime}_0)}r(s^{\prime}_0,a^{\prime}_0)\Vert_\infty\\
&\quad+ \alpha \gamma\Vert  e_{(s_0,a_0)}\max \limits_{a} q_0(s_1,a) - e_{(s^{\prime}_0,a^{\prime}_0)}\max \limits_{a}q^{\prime}_0(s^{\prime}_1,a)\Vert_\infty\\
& \leq \Vert q_0 - q^{\prime}_0\Vert_\infty + \alpha d_0(x_0^{\prime}, x_0)q_{\max} +  \alpha r_{\max} d_0(x_0^{\prime}, x_0) + \alpha\gamma \Vert q_0 - q^{\prime}_0\Vert_\infty + \alpha\gamma q_{\max} d_0(x_0^{\prime}, x_0)\\
& = (1 + \alpha\gamma) \Vert q_0 - q^{\prime}_0\Vert_\infty + (\alpha r_{\max} + \alpha(\gamma+1) q_{\max})d_0(x_0^{\prime}, x_0).
\end{aligned}$$

Therefore, we obtain
$$
\begin{aligned}
\mathbb{E}\left[d_0(x_1 , x_1^{\prime}) + \Vert q_1 - q^{\prime}_1\Vert_\infty^2\right] &=  \mathbb{E}\left[d_0(x_1 , x_1^{\prime})\right] + \mathbb{E}\left[\Vert q_1 - q^{\prime}_1\Vert_\infty^2\right]\\
& \leq  \mathbb{E}\left[d_0(x_0 , x_0^{\prime})\right] + 2(1 + \alpha\gamma)^2\mathbb{E}\left[ \Vert q_0 - q^{\prime}_0\Vert_\infty^2\right] + 2(\alpha r_{\max} + \alpha(\gamma+1) q_{\max})^2\mathbb{E}\left[d_0(x_0 , x_0^{\prime})\right]\\
& \leq \rho\Bar{W}^2_2\left(\mathcal{L}\left(x_0,q_0\right), \mathcal{L}(x_0^{\prime}, q_0^{\prime})\right),
\end{aligned}$$
with $\rho = \max\left(1+ 2(\alpha r_{\max} + \alpha(\gamma+1) q_{\max})^2, 2(1 + \alpha\gamma)^2\right).$
\end{proof}

By the triangle inequality of extended Wasserstein 2-distance, we obtain
\begin{equation}\label{a}
\begin{aligned}
\bar{W}_2\left(\mathcal{L}\left(x_1, q_1\right), \bar{\mu}\right) &\leq \bar{W}_2\left(\mathcal{L}\left(x_1, q_1\right), \mathcal{L}\left(x_{k+1}, q_{k+1}\right)\right)+\bar{W}_2\left(\mathcal{L}\left(x_{k+1}, q_{k+1}\right), \bar{\mu}\right)\\
&\leq \rho\Bar{W}^2_2\left(\Bar{\mu}, \mathcal{L}(x_k , q_k)\right)+\Bar{W}_2\left(\mathcal{L}\left(x_{k+1},q_{k+1}\right), \bar{\mu}\right)\\
& \stackrel{k \rightarrow \infty}{\longrightarrow} 0,
\end{aligned}   
\end{equation}
where the second inequality holds by Lemma \ref{invariance} and last step comes from the weak convergence result. Therefore, we have proved that $(x_k,q_k)_{k \geq 0}$ converge to a unique stationary distribution $\bar{\mu}.$

Next, we restate a variant of Theorem 3.1 in \cite{chen2021lyapunov} as follows without the assumption that $r_{\max} \leq 1$, which can be proved by Theorem 2.1 and 3.1 in \cite{chen2021lyapunov}.

\begin{thm}[Theorem 3.1 in \cite{chen2021lyapunov}]\label{thm:2order}
Under Assumption \ref{MC}, and $\alpha t_\alpha \leq c_0\frac{(1-\beta)^2}{\log(|\mathcal{S}\Vert\mathcal{A}|)}$ ($c_0$ is a constant), for all $k \geq t_\alpha$, we obtain
\begin{align} \label{eq:2order}
    &\mathbb{E}\left[\left\|q_k-q^*\right\|_{\infty}^2\right] \leq c_{Q, 1}\left(1-\frac{\left(1-\beta\right) \alpha}{2}\right)^{k-t_\alpha}+c_{Q} \frac{\log (|\mathcal{S}\Vert\mathcal{A}|)}{\left(1-\beta\right)^2} \alpha t_\alpha,
\end{align}
where $c_{Q, 1} = 3\left(\Vert q_0 - q^*\Vert_\infty + \Vert q_0\Vert_\infty + \frac{r_{\max}}{3}\right)^2$ and $c_{Q} = 912e\left(3\Vert q^*\Vert_\infty + r_{\max}\right).$
\end{thm}

We remark that $c_0$ is the same numerical constant as $c_{Q,0}$ appearing in Theorem 3.1 in \cite{chen2021lyapunov}.

Finally, we establish the following lemma to bound the variance of the limit random vector $q_{\infty}$, $\operatorname{Var}\left(q_{\infty}\right)$.

\begin{lem}\label{variance}
Under Assumption \ref{MC}, and $\alpha t_\alpha \leq c_0\frac{(1-\beta)^2}{\log(|\mathcal{S}\Vert\mathcal{A}|)}$ ($c_0$ is a constant), we obtain 
\[\operatorname{Var}\left(q_{\infty}\right) \leq c_{Q} \frac{\log (|\mathcal{S}\Vert\mathcal{A}|)}{\left(1-\beta\right)^2} \alpha t_\alpha\]
and
\[(\mathbb{E}[\Vert q_\infty\Vert_\infty])^2 \leq \mathbb{E}[\Vert q_\infty\Vert_\infty^2] \leq  2c_{Q}c_{0} + 2\Vert q^*\Vert^2,\]
where $c_{Q} = 912e\left(3\Vert q^*\Vert_\infty + r_{\max}\right).$
\end{lem}


\begin{proof}[Proof for Lemma \ref{variance}]
We have shown that the sequence $\left(q_k\right)_{k \geq 0}$ converges weakly to $q_\infty$ in $\mathcal{P}_2(\mathbb{R}^d)$. It is well known that weak convergence in $\mathcal{P}_2(\mathbb{R}^d)$ is equivalent to convergence in distribution and the convergence of the first two moments. As a result, we have
\begin{equation}\label{limit}
\mathbb{E}\left[\lVert q_\infty - q^*\rVert_\infty^2\right] = \lim_{k\to\infty} \mathbb{E}\left[\lVert q_k - q^*\rVert_\infty^2\right]. 
\end{equation}

Taking $k \rightarrow \infty$ on the both sides of equation \eqref{eq:2order} and combining with equation \ref{limit} yields
\[\mathbb{E}[\lVert q_\infty - q^*\rVert^2_{\infty}] \leq c_{Q} \frac{\log (|\mathcal{S}\Vert\mathcal{A}|)}{\left(1-\beta\right)^2} \alpha t_\alpha.  \]
Note that $q^*$ is a deterministic quantity. We thus have
\begin{align*}
    & \operatorname{Var}\left(q_{\infty}\right) \stackrel{(i)}{\leq} \max_{s,a} \operatorname{Var}\left(q_{\infty}(s,a)\right) \leq \mathbb{E}[\lVert q_\infty - q^*\rVert^2_{\infty} ] \leq c_{Q} \frac{\log (|\mathcal{S}\Vert\mathcal{A}|)}{\left(1-\beta\right)^2} \alpha t_\alpha,
\end{align*}
where the inequality (i) means an upper bound on elementwise $\ell_{\infty}$ norm for the covariance matrix $\operatorname{Var}\left(q_{\infty}\right)$.

In addition, we have
\[
\begin{aligned}
(\mathbb{E}[\Vert q_\infty\Vert_{\infty}])^2 &\leq \mathbb{E}[\Vert q_\infty\Vert^2_{\infty}]\\
&\leq \mathbb{E}[(\Vert q_\infty - q^*\Vert_{\infty} + \Vert q^*\Vert_{\infty})^2]\\
&\leq 2\mathbb{E}(\Vert q_\infty - q^*\Vert^2_{\infty}) + 2\Vert q^*\Vert^2_{\infty}\\
&\leq 2c_{Q} \frac{\log (|\mathcal{S}\Vert\mathcal{A}|)}{\left(1-\beta\right)^2} \alpha t_\alpha + 2\Vert q^*\Vert^2_{\infty}\\
&\leq 2c_{Q}c_{0} + 2\Vert q^*\Vert^2_{\infty}.
\end{aligned}
\]
\end{proof}

Therefore, we have proved parts 1 and 2 of Theorem \ref{limit4tabular}.

\subsection{Convergence Rate}\label{sec:proof_limit4tabular_rate}

So far we have established that the Markov chain $(x_k,q_k)_{k \geq 0}$ converges to a unique stationary distribution $\bar{\mu}\in \mathcal{P}_2(\mX\times\mathbb{R}^{|\mS||\mA|})$. As a result, $(q_k)_{k \geq 0}$ converges weakly to $\mu \in \mathcal{P}_2(\mathbb{R}^{|\mS||\mA|})$, where $\mu$ is the second marginal of $\bar{\mu}$ over $\mathbb{R}^{|\mS||\mA|}$. We next focus on the convergence rate of $(q_k)_{k \geq 0}$.

Let us consider the coupled processes defined as equation \eqref{couple} in Section~\ref{sec:proof_limit4tabular_coupling}. Suppose that the initial iterate $(x_0, q_0^{[2]})$ follows the stationary distribution $\Bar{\mu}$, thus $\mathcal{L}(x_k, q_k^{[2]}) = \Bar{\mu}$ and $\mathcal{L}(q_k^{[2]}) = \mu$ for all $k \geq 0$. By equation \eqref{geo}, we have for all $k \geq 0:$
\begin{equation}\label{eq:convergencerate}
\begin{aligned}
W_2^2\left(\mathcal{L}(q_k^{[1]}), \mu\right) &= W_2^2\left(\mathcal{L}(q_k^{[1]}), \mathcal{L}(q_k^{[2]})\right)\\
&\leq \Bar{W}_2^2\left(\mathcal{L}(x_k,q_k^{[1]}), \mathcal{L}(x_k,q_k^{[2]})\right)\\
&\leq 12\mathbb{E}\left[\Vert q_0^{[1]}-q_0^{[2]}\Vert_\infty^2\right] \left(1-\frac{\left(1-\beta\right) \alpha}{2}\right)^{k-t_\alpha}\\
&\leq 24 \left(1-\frac{\left(1-\beta\right) \alpha}{2}\right)^{k-t_\alpha} \cdot\left(\mathbb{E}\left[\Vert q_0^{[1]}\Vert_\infty^2\right] + \mathbb{E}\left[\Vert q_\infty^{[1]}\Vert_\infty^2\right]\right).
\end{aligned}
\end{equation}

Here the last step follows from the fact that $(x_0, q_0^{[2]})$ follows the stationary distribution, and thus $\mathbb{E}\left[\Vert q_0^{[2]}\Vert_\infty^2\right] = \mathbb{E}\left[\Vert q_\infty^{[2]}\Vert_\infty^2\right] = \mathbb{E}\left[\Vert q_{\infty}^{[1]}\Vert_\infty^2\right].$

We have completed the proof of Theorem \ref{limit4tabular}.

\subsection{Proof of Proposition \ref{w}} \label{sec:proof_w}

To analyze the convergence rate of $w_k$, we construct two new sequences $\{\underline{w}_k\}_{k \geq 0}$ and $\{\Bar{w}_k\}_{k \geq 0}$ that satisfy the following recursion:
\begin{equation*}
\begin{aligned}
    \underline{w}_{k+1}(s_k, a_k) &= (1-\alpha)\underline{w}_k(s_k, a_k) + \alpha\gamma \left( \min \limits_{a^{\prime}} \underline{w}_k (s_{k+1}, a^{\prime})  \right), \\
    \Bar{w}_{k+1}(s_k, a_k) &= (1-\alpha)\Bar{w}_k(s_k, a_k) + \alpha\gamma \left( \max \limits_{a^{\prime}} \Bar{w}_k (s_{k+1}, a^{\prime})  \right).
\end{aligned}
\end{equation*}
Let $\underline{w}_0 = w_0 = \Bar{w}_0$. We then prove that $\underline{w}_k$ and $ \Bar{w}_k$ provide a lower bound and upper bound for $w_k$, respectively. 

\begin{lem}\label{bound4w}
For all $k \geq 0$ and all $ (s,a) \in \mS\times \mA$, $\underline{w}_k(s,a) \leq w_k(s,a) \leq \Bar{w}_k(s,a)$.
\end{lem}

\begin{proof}[Proof of Lemma \ref{bound4w}]
We use an inductive argument to prove this lemma. 

For $k = 0$, $\underline{w}_0 = w_0 = \Bar{w}_0$ by definition.

Now assume for $k = k_0$, $\underline{w}_{k_0} \leq w_{k_0} \leq \Bar{w}_{k_0}$. For $k = k_0 + 1$, we consider the following two cases:

For $(s,a) \neq (s_{k_0}, a_{k_0})$, we have
\[
\underline{w}_{k_0+1}(s,a) = \underline{w}_{k_0}(s,a) \leq w_{k_0}(s,a) = w_{k_0+1}(s,a) \leq \Bar{w}_{k_0}(s,a) = \Bar{w}_{k_0+1}(s,a).
\]

For $(s,a) = (s_{k_0}, a_{k_0})$, we have
\[
\begin{aligned}
w_{k_0+1}(s, a) &= (1-\alpha)w_{k_0}(s, a) + \alpha\gamma \left( \max \limits_{a^{\prime}} q_{k_0}^{[1]}(s_{k_0+1}, a^{\prime}) - \max \limits_{a^{\prime}} q_{k_0}^{[2]}(s_{k_0+1}, a^{\prime}) \right) \\
    &\leq (1-\alpha)w_{k_0}(s, a) + \alpha\gamma \max \limits_{a^{\prime}} \left( q_{k_0}^{[1]}(s_{k_0+1}, a^{\prime}) - q_{k_0}^{[2]}(s_{k_0+1}, a^{\prime}) \right) \\
    &= (1-\alpha)w_{k_0}(s, a) + \alpha\gamma \max \limits_{a^{\prime}} \left( w_{k_0}(s_{k_0+1}, a^{\prime})  \right) \\
    &\leq (1-\alpha)\Bar{w}_{k_0}(s, a) + \alpha\gamma \max \limits_{a^{\prime}} \left( \Bar{w}_{k_0}(s_{k_0+1}, a^{\prime})  \right) = \Bar{w}_{k_0+1}(s,a). \\
    w_{k_0+1}(s, a) &= (1-\alpha)w_{k_0}(s, a) + \alpha\gamma \left( \max \limits_{a^{\prime}} q_{k_0}^{[1]}(x_{k_0+1}, a^{\prime}) - \max \limits_{a^{\prime}} q_{k_0}^{[2]}(s_{k_0+1}, a^{\prime}) \right) \\
    &\geq (1-\alpha)w_{k_0}(s, a) + \alpha\gamma \min \limits_{a^{\prime}} \left( q_{k_0}^{[1]}(s_{k_0+1}, a^{\prime}) - q_{k_0}^{[2]}(s_{k_0+1}, a^{\prime}) \right) \\
    &= (1-\alpha)w_{k_0}(s, a) + \alpha\gamma \min \limits_{a^{\prime}} \left( w_{k_0}(s_{k_0+1}, a^{\prime})  \right) \\
    &\geq (1-\alpha)\underline{w}_{k_0}(s, a) + \alpha\gamma \min \limits_{a^{\prime}} \left( \underline{w}_{k_0}(s_{k_0+1}, a^{\prime})  \right) = \underline{w}_{k_0+1}(s,a) .
\end{aligned}
\]
By induction, we complete the proof of Lemma \ref{bound4w}.

\end{proof}

Notice that $\{-\underline{w}_k\}$ and $\{ \Bar{w}_k \}$ can be viewed as the iterates generated by the Q-learning algorithm with $r(s,a) = 0$ for all $(s,a)$. Then, for both $\{-\underline{w}_k\}$ and $\{ \Bar{w}_k \}$, we obtain the following bound for the second moment of $\underline{w}_k$ and $ \Bar{w}_k $ by Theorem \ref{thm:2order} with the special case of $q^* = 0$ and $r_{\max} = 0$.
\[
\begin{aligned}
\mathbb{E}\left[\left\|\underline{w}_k\right\|_{\infty}^2\right] &\leq 12\mathbb{E}\left[\Vert w_0\Vert_\infty^2\right]\left(1-\frac{\left(1-\beta\right) \alpha}{2}\right)^{k-t_\alpha},\\
\mathbb{E}\left[\left\|\Bar{w}_k\right\|_{\infty}^2\right] &\leq 12\mathbb{E}\left[\Vert w_0\Vert_\infty^2\right]\left(1-\frac{\left(1-\beta\right) \alpha}{2}\right)^{k-t_\alpha}.
\end{aligned}
\]

By Lemma \ref{bound4w}, the same bound can also be applied to $\mathbb{E}\left[\left\|w_k\right\|_{\infty}^2\right].$ We thus have
\[\mathbb{E}\left[\left\|w_k\right\|_{\infty}^2\right] \leq 12\mathbb{E}\left[\Vert w_0\Vert_\infty^2\right] \left(1-\frac{\left(1-\beta\right) \alpha}{2}\right)^{k-t_\alpha}.\]

\subsection{Proof of Corollary \ref{co4limit}} \label{sec:proof_co4limit}

Lemma \ref{variance} states that the second moment of $q_\infty$ is bounded by a constant, which is $\mathbb{E}\left[\Vert q_\infty\Vert_\infty^2\right] = \mathcal{O}(1)$. Combining this bound with equation \eqref{convergerate} in Theorem \ref{limit4tabular}, we obtain
\[W^2_2(\mathcal{L}(q_k), \mu) \leq C(r,\gamma,P)\cdot \left(1-\frac{\left(1-\beta\right) \alpha}{2}\right)^{k-t_\alpha}, \]
where $C(r,\gamma,P)$ is a numerical constant that only depends on the reward function $r$, discounted factor $\gamma$, and stationary distribution for Markov chain $(x_k)_{k \geq 0}$.

By \cite[Theorem 4.1]{villani2009optimal}, there exists a coupling between $q_k$ and $q_\infty$ such that 
\[W^2_2(\mathcal{L}(q_k), \mu) = \mathbb{E}[\Vert q_k - q_\infty\Vert_{\infty}^2].\]

By the above bounds and applying Jensen's inequality twice, we obtain that 
\[
\begin{aligned}
\Vert\mathbb{E}[q_k - q_\infty]\Vert_{\infty}^2 &\leq \left(\mathbb{E}[\Vert q_k - q_\infty\Vert_{\infty}]\right)^2\\
&\leq \mathbb{E}[\Vert q_k - q_\infty\Vert_{\infty}^2]\\
&\leq C(r,\gamma,P)\left(1-\frac{\left(1-\beta\right) \alpha}{2}\right)^{k-t_\alpha}.
\end{aligned}
\]

We thus have for all $k \geq t_\alpha$,
\[\Vert\mathbb{E}[q_k] - \mathbb{E}[q_\infty]\Vert_{\infty} \leq \mathbb{E}[\Vert q_k - q_\infty\Vert_{\infty}] \leq C(r,\gamma,P)\left(1-\frac{\left(1-\beta\right) \alpha}{2}\right)^{\frac{k-t_\alpha}{2}}.\]

For the second moment, we notice that
\begin{equation}\label{secondmoment}
\begin{aligned}
& \left\|\mathbb{E}\left[q_k q_k^{\top}\right]-\mathbb{E}\left[q_{\infty} q_{\infty}^{\top}\right]\right\|_\infty \\
& =\left\|\mathbb{E}\left[\left(q_k-q_{\infty}+q_{\infty}\right)\left(q_k-q_{\infty}+q_{\infty}\right)^{\top}\right]-\mathbb{E}\left[q_{\infty} q_{\infty}^{\top}\right]\right\|_\infty \\
& =\left\|\mathbb{E}\left[\left(q_k-q_{\infty}\right)\left(q_k-q_{\infty}\right)^{\top}\right]+\mathbb{E}\left[q_{\infty}\left(q_k-q_{\infty}\right)^{\top}\right]+\mathbb{E}\left[\left(q_k-q_{\infty}\right) q_{\infty}^{\top}\right]\right\|_\infty \\
& \leq\left\|\mathbb{E}\left[\left(q_k-q_{\infty}\right)\left(q_k-q_{\infty}\right)^{\top}\right]\right\|_\infty+\left\|\mathbb{E}\left[q_{\infty}\left(q_k-q_{\infty}\right)^{\top}\right]\right\|_\infty+\left\|\mathbb{E}\left[\left(q_k-q_{\infty}\right) q_{\infty}^{\top}\right]\right\|_\infty \\
& \leq \mathbb{E}\left[\left\|\left(q_k-q_{\infty}\right)\left(q_k-q_{\infty}\right)^{\top}\right\|_\infty\right]+\mathbb{E}\left[\left\|q_{\infty}\left(q_k-q_{\infty}\right)^{\top}\right\|_\infty\right]+\mathbb{E}\left[\left\|\left(q_k-q_{\infty}\right) q_{\infty}^{\top}\right\|_\infty\right] \\
& \leq \mathbb{E}\left[\left\|q_k-q_{\infty}\right\|_\infty^2\right]+2 \mathbb{E}\left[\left\|q_{\infty}^{\top}\left(q_k-q_{\infty}\right)\right\|_\infty\right] \\
& \leq \mathbb{E}\left[\left\|q_k-q_{\infty}\right\|_\infty^2\right]+2\left(\mathbb{E}\left[\left\|q_k-q_{\infty}\right\|_\infty^2\right] \mathbb{E}\left[\left\|q_{\infty}\right\|_\infty^2\right]\right)^{1 / 2}.
\end{aligned}
\end{equation}

Meanwhile, we have
\[\mathbb{E}\left[\left\|q_k-q_{\infty}\right\|_\infty^2\right] \leq C(r,\gamma,P)\left(1-\frac{\left(1-\beta\right) \alpha}{2}\right)^{k-t_\alpha} \quad \text{and} \quad \mathbb{E}\left[\left\|q_{\infty}\right\|_\infty^2\right] = \mathcal{O}(1).\]

Substituting the above bounds into the right-hand side of inequality \eqref{secondmoment} yields
$$
\left\|\mathbb{E}\left[q_k q_k^{\top}\right]-\mathbb{E}\left[q_{\infty} q_{\infty}^{\top}\right]\right\|_\infty \leq C^{\prime}(r,\gamma,P)\left(1-\frac{\left(1-\beta\right) \alpha}{2}\right)^{\frac{k-t_\alpha}{2}},
$$
thereby completing the proof for Corollary \ref{co4limit}.

\section{Proof of Theorem \ref{thm:clt}} \label{sec:proof_clt}

Define $f: \mathcal{X} \times \mathbb{R}^{|\mathcal{S}||\mathcal{A}|} \to \mathbb{R}^{|\mathcal{S}||\mathcal{A}|}$, such that $f(x,q) := q - \mathbb{E}(q_\infty)$. Consider $\{(x_k,q_k)\}_{k \geq 0}$ with $x_0 \sim \mu_{\mathcal{X}}$ and $q_0 \sim \bar{\mu}(\cdot \mid x_0
)$.

$$
\begin{aligned}
\left\|\sum_{k=0}^{n-1} P^k f\right\|_{\infty, L_{\bar{\mu}}^2} & =  \sqrt{ \mathbb{E}_{(x_0,q_0) \sim \bar{\mu}} \|\sum_{k=0}^{n-1}\mathbb{E}[f(x_k,q_k) \mid x_0, q_0]\|_\infty^2}\\
& = \sqrt{ \mathbb{E}_{(x_0,q_0) \sim \bar{\mu}} \|\sum_{k=0}^{n-1}\mathbb{E}[q_k\mid x_0, q_0] - n\mathbb{E}(q_\infty)\|_\infty^2}\\
& \leq \sqrt{ \mathbb{E}_{(x_0,q_0) \sim \bar{\mu}} \|\sum_{k=0}^{n-1}\mathbb{E}[q_k\mid x_0, q_0] - n\mathbb{E}(q_\infty)\|_2^2}\\
& = \sqrt{ \mathbb{E}_{(x_0,q_0) \sim \bar{\mu}} \sum_{i,j=0}^{n-1} \mathbb{E}[q_i - \mathbb{E}(q_\infty) \mid x_0, q_0]^T\mathbb{E}[q_j - \mathbb{E}(q_\infty) \mid x_0, q_0]}.
\end{aligned}
$$

Define $g_k(x,q) := \mathbb{E}[q_k - \mathbb{E}(q_\infty) \mid (x_0, q_0) = (x,q)]$, we then give the following Lemma \ref{lemma:g} to uniformly bound $g_k(x,q)$ for all $(x,q) \in \mathcal{X} \times \mathbb{R}^{|\mathcal{S}|| \mathcal{A}|}$. The proof of Lemma \ref{lemma:g} is given at section \ref{sec:proof of g}.

\begin{lem}\label{lemma:g}
For all $(x,q) \in \mathcal{X} \times \mathbb{R}^{|\mathcal{S}|| \mathcal{A}|}$, when $k \geq t_\alpha$, there exist two constant $\lambda_0,\lambda_1$ such that
\[\|g_k(x,q)\|_2 \leq \lambda_0 \cdot \lambda_1^{k},\]
where $\lambda_0 > 0 $ and $0 < \lambda_1 < 1$.
\end{lem}

By Lemma \ref{lemma:g}, we obtain
$$
\begin{aligned}
\left\|\sum_{k=0}^{n-1} P^k f\right\|_{\infty, L_{\bar{\mu}}^2} & \leq \sqrt{ \mathbb{E}_{(x_0,q_0) \sim \bar{\mu}} \sum_{i,j=0}^{n-1} \mathbb{E}[q_i - \mathbb{E}(q_\infty) \mid x_0, q_0]^T\mathbb{E}[q_j - \mathbb{E}(q_\infty) \mid x_0, q_0]}\\
& \leq \sqrt{ \mathbb{E}_{(x_0,q_0) \sim \bar{\mu}} \sum_{i,j=0}^{n-1} \|g_i(x_0,q_0)\|_2 \|g_j(x_0,q_0)\|_2}\\
& \leq \sqrt{\frac{\lambda_0^2}{(1-\lambda_1)^2}} = \mathcal{O}(1).
\end{aligned}
$$

By Lemma \ref{variance}, we can observe that $\int\|f(x,q)\|_\infty^2 \bar{\mu}(\mathrm{d} (x,q))<\infty$ and $\int f(x) \bar{\mu}(\mathrm{d} (x,q))=\mathbf{0}$. Therefore, by Theorem 2.1 in \cite{xie2022statistical}, we complete the proof for Theorem \ref{thm:clt}.

\subsection{Proof of Lemma \ref{lemma:g}}\label{sec:proof of g}

Recall that the Markov chain $\{x_k\}_{k \geq 0}$ mixes geometrically fast to the stationary distribution $\mu_{\mathcal{X}}$, and there exist $c \geq 0$ and $\rho \in (0,1)$ s.t.
\[
\max \limits_{x \in \mathcal{X}} \Vert p^k(x,\cdot) - \mu_{\mathcal{X}}(\cdot)\Vert_{TV} \leq c \rho^k,
\]
When $k \geq t_\alpha$, we have
\[
\begin{aligned}
g_k(x,q) &= \sum_{x^{\prime} \in \mathcal{X}}\int_{q^{\prime} \in \mathbb{R}^{|\mathcal{S}||\mathcal{A}|}}\mathbb{E}[q_k - \mathbb{E}(q_\infty) \mid (x_{\lfloor\frac{k}{2}\rfloor}, q_{\lfloor\frac{k}{2}\rfloor}) = (x,q)]\mathrm{d}\mathbb{P}\left((x_{\lfloor\frac{k}{2}\rfloor}, q_{\lfloor\frac{k}{2}\rfloor}) = (x^{\prime}, q^{\prime}) \mid(x_0, q_0) = (x,q)\right)\\
& = \sum_{x^{\prime} \in \mathcal{X}}\int_{q^{\prime} \in \mathbb{R}^{|\mathcal{S}||\mathcal{A}|}}g_{k - \lfloor\frac{k}{2}\rfloor}(x^{\prime}, q^{\prime})\mathrm{d}\mathbb{P}\left((x_{\lfloor\frac{k}{2}\rfloor}, q_{\lfloor\frac{k}{2}\rfloor}) = (x^{\prime}, q^{\prime}) \mid(x_0, q_0) = (x,q)\right)\\
& = \sum_{x^{\prime} \in \mathcal{X}}\int_{q^{\prime} \in \mathbb{R}^{|\mathcal{S}||\mathcal{A}|}}g_{k - \lfloor\frac{k}{2}\rfloor}(x^{\prime}, q^{\prime})\underbrace{\mathbb{P}\left(x_{\lfloor\frac{k}{2}\rfloor} = x^{\prime} \mid (x_0, q_0) = (x,q)\right)}_{p(x^{\prime})}\mathrm{d}\underbrace{\mathbb{P}\left(q_{\lfloor\frac{k}{2}\rfloor} =  q^{\prime}  \mid x_{\lfloor\frac{k}{2}\rfloor} = x^{\prime}, (x_0, q_0) = (x,q)\right)}_{\eta(q^{\prime} \mid x^{\prime})}\\
&=  \underbrace{\sum_{x^{\prime} \in \mathcal{X}}\int_{q^{\prime} \in \mathbb{R}^{|\mathcal{S}||\mathcal{A}|}}g_{k - \lfloor\frac{k}{2}\rfloor}(x^{\prime}, q^{\prime})\mu_{\mathcal{X}}(x^{\prime})\mathrm{d}\eta(q^{\prime} \mid x^{\prime})}_{T_1} + \underbrace{\sum_{x^{\prime} \in \mathcal{X}}\int_{q^{\prime} \in \mathbb{R}^{|\mathcal{S}||\mathcal{A}|}}g_{k - \lfloor\frac{k}{2}\rfloor}(x^{\prime}, q^{\prime})(p(x^{\prime}) - \mu_{\mathcal{X}}(x^{\prime}))\mathrm{d}\eta(q^{\prime} \mid x^{\prime})}_{T_2}.
\end{aligned}
\]

By Corollary \ref{co4limit}, when $x_0 \sim \mu_{\mathcal{X}}$, for all $k \geq t_\alpha$ and arbitrary $q_0$ we have

\[\Vert\mathbb{E}[q_k] - \mathbb{E}[q_\infty]\Vert_{\infty} \leq  C(r,\gamma,P)\left(1-\frac{\left(1-\beta\right) \alpha}{2}\right)^{\frac{k-t_\alpha}{2}}.\]

Therefore, we obtain

\[
\begin{aligned}
\|T_1\|_2 &\leq \sqrt{|\mathcal{S}|| \mathcal{A}|}\|T_1\|_\infty\\
& = \sqrt{|\mathcal{S}|| \mathcal{A}|}\|\mathbb{E}_{x^{\prime} \sim \mu_{\mathcal{X}}, q^{\prime}\sim \eta(q^{\prime} \mid x^{\prime})}g_{k - \lfloor\frac{k}{2}\rfloor}\|_\infty\\
&\leq  \sqrt{|\mathcal{S}|| \mathcal{A}|}C(r,\gamma,P)\left(1-\frac{\left(1-\beta\right) \alpha}{2}\right)^{\frac{k - \lfloor\frac{k}{2}\rfloor-t_\alpha}{2}}\\
&\leq \left(\sqrt{|\mathcal{S}|| \mathcal{A}|}C(r,\gamma,P)\left(1-\frac{\left(1-\beta\right) \alpha}{2}\right)^{\frac{-t_\alpha}{2}}\right)\left(1-\frac{\left(1-\beta\right) \alpha}{2}\right)^{\frac{k}{2}}.
\end{aligned}\]

Note that $\|q_k\|_\infty \leq q_{\max}$, $\|g_k(x,q)\|_\infty \leq 2q_{\max},$ we have
\[\|T_2\|_2  \leq \sqrt{|\mathcal{S}|| \mathcal{A}|}\|T_2\|_\infty \leq \sqrt{|\mathcal{S}|| \mathcal{A}|} c\rho^{\lfloor\frac{k}{2}\rfloor}|\mathcal{S}|^2| \mathcal{A}| = |\mathcal{S}|^\frac{5}{2}| \mathcal{A}|^\frac{3}{2}c\rho^{\lfloor\frac{k}{2}\rfloor}
\leq |\mathcal{S}|^\frac{5}{2}| \mathcal{A}|^\frac{3}{2}c\rho^{-1}\rho^{\frac{k}{2}}.\]

Therefore, we have

\[
\begin{aligned}
\|g_k(x,q)\|_2 &= \|T_1 + T_2\|_2\\
&\leq \|T_1  \|_2+\|  T_2\|_2\\
&\leq \left(\left(\sqrt{|\mathcal{S}|| \mathcal{A}|}C(r,\gamma,P)\left(1-\frac{\left(1-\beta\right) \alpha}{2}\right)^{\frac{-t_\alpha}{2}}\right) + |\mathcal{S}|^\frac{5}{2}| \mathcal{A}|^\frac{3}{2}c\rho^{-1}\right) \left(\max\left\{\sqrt{\left(1-\frac{\left(1-\beta\right) \alpha}{2}\right)},\sqrt{\rho}\right\}\right)^k.
\end{aligned}
\]

Let $\lambda_0 = \left(\left(\sqrt{|\mathcal{S}|| \mathcal{A}|}C(r,\gamma,P)\left(1-\frac{\left(1-\beta\right) \alpha}{2}\right)^{\frac{-t_\alpha}{2}}\right) + |\mathcal{S}|^\frac{5}{2}| \mathcal{A}|^\frac{3}{2}c\rho^{-1}\right)$ and $\lambda_1  = \max\left\{\sqrt{\left(1-\frac{\left(1-\beta\right) \alpha}{2}\right)},\sqrt{\rho}\right\}$, we complete the proof of Lemma \ref{lemma:g}.

\section{Proof of Theorem \ref{thm:bias}} \label{sec:proof_bias}

In this section, we prove Theorem \ref{thm:bias} on the characterization of the bias $\mathbb{E}(q_\infty) - q^*$. The proof consists of five steps, which are given in the following five sub-sections.

\subsection{Step 1: Local linearization of Operator $F$}\label{sec:linearization}

Unlike linear SA, the operator $F$ in the update rule of Q-learning (cf.\ equation \eqref{eq:tabular}) is nonlinear and nonsmooth, which makes the analysis considerably more challenging. To address this issue, we employ the local linearization of the operator $F$ around the optimal solution $q^*,$ with a higher order remaining term as stated in Proposition \ref{gradient} and \ref{44tabular}. We provide complete proof here. 

\begin{proof}[Proof of Proposition~\ref{gradient}]
 Recall that we define the unique optimal action with respect to the optimal Q-function $q^*$ as 
 \[a^*_s := \text{arg}\max\limits_{a} q^*(s,a).\]

We define a function $G_{q^*}:\mX \rightarrow \R^{|\mS||\mA|\times |\mS||\mA|}$ as follows: for each $ x = (s_0,a_0,s_1)\in \mathcal{X}$, 
 $$\left[G_{q^*}(x)\right]\left[(s,a),(\Bar{s}, \Bar{a})\right] = 
 \begin{cases}
 1, &  (s,a) = (\Bar{s}, \Bar{a}) \neq (s_0,a_0)\\
 \gamma, & (s,a) = (s_0,a_0), (\Bar{s}, \Bar{a}) = (s_1, a^*_{s_1})\\
 0, & \text{otherwise}.
 \end{cases}$$
Note that the operator $F(x,\cdot)$ is nonsmooth and does not admit any gradient. On the other hand, by the uniqueness of the optimal policy $\pi^*,$ we can locally linearize $F(x,\cdot)$ around $q^*$. In particular, $G_{q^*}(x) - I_d$ serves as an approximate "gradient" of the operator $F(x,\cdot)$ around $q^*$. Define
$$R(x, q) = F(x, q) - F(x, q^*) - (G_{q^*}(x) - I_d)(q-q^*).$$
We can observe that for $\forall (s,a) \neq (s_0,a_0)$, $\left[R(x,q)\right](s,a) = 0.$
For $(s,a) = (s_0,a_0)$, we have
$$
\begin{aligned}
\left[R(x,q)\right](s_0,a_0) & = \gamma\left(\max \limits_{a}q(s_1, a) - q(s_1, a^*_{s_1})\right) \geq 0.
\end{aligned}$$

If $\Vert q - q^*\Vert_\infty < \Delta$, by Assumption \ref{assum:smooth}, for any action $a \neq a^*_{s_1}$, we have
$$
\begin{aligned}
q(s_1, a^*_{s_1}) & > q^*(s_1, a^*_{s_1}) - \delta\\
& \geq  q^*(s_1, a) + \delta\\
& > q(s_1,a).
\end{aligned}$$
Thus,
$$\left[R(x, q)\right](s_0,a_0) = \gamma\left(\max \limits_{a}q(s_1, a) - q(s_1, a^*_{s_1})\right) = 0.$$ 

If $\Vert q - q^*\Vert_\infty \geq \Delta$, we have
$$
\begin{aligned}
|\left[R(x,q)\right](s_0,a_0)| & = \gamma|\max \limits_{a}q(s_1, a) - q(s_1, a^*_{s_1})|\\
& = \gamma|\max \limits_{a}q(s_1, a) -\max \limits_{a}q^*(s_1, a) + q^*(s_1, a^*_{s_1})- q(s_1, a^*_{s_1})|\\
& \leq 2\gamma \Vert q - q^*\Vert_\infty\\
& \leq \frac{2\gamma}{\Delta^3} \Vert q - q^*\Vert_\infty^4.
\end{aligned}$$

Combining the two situations considered above, we finally obtain that 
\[\Vert R(x, q)\Vert_\infty\leq \frac{2\gamma}{\Delta^3} \Vert q - q^*\Vert_\infty^4.\]
which proves the first part of Proposition \ref{gradient}. 

For the second part, we can multiply the $G_{q^*}(x)$ by an arbitrary nonzero vector $H \in \mathbb{R}^{|\mathcal{S}\Vert\mathcal{A}|}$. Let $(s_h,a_h) = \text{arg}\max \limits_{(s,a) \in \mathcal{S} \times \mathcal{A}} H(s,a)$ and $p_h = \mu_{\mathcal{S}}(s_h,a_h) $. By Assumption \ref{MC}, $p_h > 0$. Without loss of generality, we can assume $H(s_h,a_h) >0$, otherwise we can replace $ H$ with $-H$.  We then have
$$
\begin{aligned}
\mathbb{E}\left[G_{q^*}^{\prime}(x)H\right](s_h,a_h) & = \gamma p_h \mathbb{E}\left(H(s_1, a^*_{s_1}) \mid s_0 = s_h, a_0 = a_h\right) + (1-p_h) H(s_h,a_h)\\
& \leq \gamma p_h H(s_h,a_h) + (1-p_h) H(s_h,a_h)\\
& < H(s_h,a_h),
\end{aligned} $$
where the second step hold as the definition of $(s_h,a_h)$ uses the maximum.

We thus have 
\[\mathbb{E}[G_{q^*}(x)]H=\mathbb{E}\left[G_{q^*}(x)H\right] \neq H.\]
As $H$ is an arbitrary vector, we conclude that $\mathbb{E}(G_{q^*}(x))$ does not have an eigenvalue of 1, thereby completing the proof for Proposition \ref{gradient}.
\end{proof}

\begin{proof}[Proof of Proposition \ref{44tabular}:]

Let $f(z) = \frac{1}{2}\Vert z\Vert_\infty^2$ and $g(z) = \frac{1}{2}\Vert z\Vert _2^2$. Note that $g(\cdot)$ is a convex, differentiable, and 1-smooth function. In Proposition \ref{44tabular}, we work with a finite demensional space $\mathbb{R}^{|\mathcal{S}\Vert \mathcal{A}|}$. By Cauchy-Schwarz Inequality, $\frac{1}{\sqrt{|\mathcal{S}\Vert \mathcal{A}|}}\Vert \cdot\Vert _2 \leq \Vert \cdot\Vert _\infty \leq \Vert \cdot\Vert _2$. We construct the Generalized Moreau Envelope of $f(\cdot)$ with respect to $g(\cdot)$ as follows:
\begin{equation*}
    M_f^{\eta,g}(z) = \min \limits_{u \in \mathbb{R}^{|\mathcal{S}\Vert \mathcal{A}|}} \left\{f(u) + \frac{1}{\eta}g(z - u)\right\},
\end{equation*}
where $\eta>0.$ For the ease of exposition, we use $M(\cdot)$ to denote $M_f^{\eta,g}(\cdot)$. We  restate  Lemma 2.1 in \cite{chen2020finite} below on the properties of $M(\cdot).$

\begin{lem}[Lemma 2.1 in \cite{chen2020finite}]\label{lemma:M}
For given $\eta > 0$. Then $M(\cdot)$ constructed above has the following properties: 
\begin{enumerate}[leftmargin=*]
    \item (\text{Smoothness}) $M(\cdot)$ is convex,  $\frac{1}{\eta}$-smooth with respect to $\| \cdot\| _2.$
    \item There exists a norm $\Vert \cdot\Vert _m$ such that $M(z) = \frac{1}{2}\Vert z\Vert _m^2$. Furthermore, there exist $l_m, u_m >0$, such that $l_m\Vert \cdot\Vert _m \leq \Vert \cdot\Vert _\infty \leq u_m\Vert \cdot\Vert _m$. Specifically, we can let $l_m = (1 + \frac{\eta}{|\mathcal{S}\Vert\mathcal{A}|})^{1/2}$, $u_m = (1 + \eta)^{1/2}$.
\end{enumerate}

\end{lem}

Therefore, $M(\cdot)$ serves as a smooth approximation of the non-smooth function $f(\cdot).$ Then, we have for $\forall k \geq 0:$

\begin{equation}\label{decomposition4M}
\begin{aligned}
&\quad M^2(q_{k+1} - q^*) \\
&\stackrel{(a)}{\leq} \left(M(q_k - q^*) + \langle \nabla M(q_k - q^*), q_{k+1} - q_k\rangle + \frac{1}{2\eta}\Vert q_{k+1} - q_k\Vert _2^2\right)^2\\
&\stackrel{(b)}{=} \left(M(q_k - q^*) + \alpha \langle \nabla M(q_k - q^*), F(x_k,q_k)\rangle + \frac{\alpha^2}{2\eta}\Vert F(x_k,q_k)\Vert _2^2\right)^2\\
&= (M(q_k - q^*) + \alpha \langle \nabla M(q_k - q^*), \Bar{F}(q_k)\rangle \\
&\quad+ \alpha \langle \nabla M(q_k - q^*), F(x_k,q_k) - \Bar{F}(q_k)\rangle + \frac{\alpha^2}{2\eta}\Vert F(x_k,q_k)\Vert _2^2)^2\\
&\stackrel{(c)}{\leq} M^2(q_{k} - q^*) + \underbrace{ 2\alpha M(q_k - q^*)\langle \nabla M(q_k - q^*), \Bar{F}(q_k)\rangle}_{T_1}\\
&\quad+\underbrace{ 2\alpha M(q_k - q^*) \langle \nabla M(q_k - q^*), F(x_k,q_k) - \Bar{F}(q_k)\rangle}_{T_2} + \underbrace{\frac{\alpha^2}{\eta} M(q_k - q^*)\Vert F(x_k,q_k)\Vert _2^2}_{T_3}\\
&\quad+ \underbrace{3\alpha^2 \langle \nabla M(q_k - q^*), \Bar{F}(q_k)\rangle^2}_{T_4} +  3\alpha^2 \underbrace{ \langle \nabla M(q_k - q^*), F(x_k,q_k) - \Bar{F}(q_k)\rangle^2}_{T_5}\\
& \quad+ \underbrace{\frac{3\alpha^4}{4\eta^2} \Vert F(x_k,q_k)\Vert _2^4}_{T_6},
\end{aligned}
\end{equation}
where $(a)$ follows from the smoothness of $M(\cdot)$ in Lemma \ref{lemma:M}, $(b)$ follows from the update rule of $q_k$ in \eqref{eq:tabular}, and $(c)$ holds by the inequality $(x+y+z)^2\leq 3(x^2+y^2+z^2).$ 

Next we derive an upper bound on $M(\cdot)^2$ by bounding each term of $T_1-T_6$.

\begin{lem}\label{T1}
For all $k \geq 0$, $ T_1 \leq -4\alpha \left(1 - \gamma \right)^2 M^2(q_k - q^*)$. 
\end{lem}

\begin{proof}[Proof of Lemma \ref{T1}] By Proposition 2.1 in \cite{chen2020finite}, we have that
$$\langle \nabla M(q_k - q^*), \Bar{F}(q_k)\rangle\leq -2\left(1 - \gamma \left(\frac{1 + \eta\sqrt{|\mathcal{S}\Vert \mathcal{A}|}}{1 + \eta}\right)^{\frac{1}{2}}\right)M(q_k - q^*).$$ 
We can always choose a sufficiently small $\eta$ such that $ \left(\frac{1 + \eta\sqrt{|\mathcal{S}\Vert \mathcal{A}|}}{1 + \eta}\right)^{\frac{1}{2}} \leq  2-\gamma$ because $\gamma < 1$, which is equivalent to
 \begin{equation}\label{eta}
 \eta \leq \frac{(2-\gamma)^2 - 1}{\sqrt{|\mathcal{S}\Vert \mathcal{A}|} - 1}.
 \end{equation}
 Since $M(\cdot)$ is non-negative, we complete the proof by multiplying $ 2\alpha M(q_k - q^*)$ on both sides.
\end{proof}

By Lemma \ref{T1}, $T_1$ can give us a desired negative drift term of order $-\mathcal{O}(\alpha)$.\\

By Cauchy-Schwarz Inequality, we can bound $T_2$ by two terms. One term is proportional to $ M^2(q_k - q^*)$ but still keep the negative drift generated by $T_1$ and the other term is proportional to $T_5$:

\begin{equation*}
\begin{aligned}
T_2 & \leq \alpha (1 - \gamma)^2M^2(q_k - q^*) + \alpha(1 - \gamma)^{-2} \langle \nabla M(q_k - q^*), F(x_k,q_k) - \Bar{F}(q_k)\rangle^2\\
& =  \alpha (1 - \gamma)^2M^2(q_k - q^*) + \alpha(1 - \gamma)^{-2} T_5.
\end{aligned}
\end{equation*}

Then, we can simplify equation \eqref{decomposition4M} as follows when $3\alpha \leq  (1 - \gamma)^{-2}$: 
\begin{equation*}
\begin{aligned}
M^2(q_{k+1} - q^*) &\leq (1 -3\alpha (1 - \gamma)^{2})M^2(q_{k} - q^*) + T_3 + T_4 +  2\alpha(1 - \gamma)^{-2} T_5 + T_6.
\end{aligned}
\end{equation*}

By Cauchy-Schwarz Inequality and Lemma A.5 in \cite{chen2021lyapunov}, $T_3$ can be bounded as follow 
\begin{equation*}
\begin{aligned}
T_3 &= \frac{\alpha^2}{\eta}M(q_k - q^*)\Vert F(x_k,q_k)\Vert_2^2\\
&\leq \frac{\alpha^2}{\eta}M(q_k - q^*)\left(36u_m^2 |\mathcal{S}\Vert\mathcal{A}|M(q_k - q^*) + 2 |\mathcal{S}\Vert\mathcal{A}|(3\Vert q^*\Vert_\infty + r_{\max})^2\right)\\
&= \frac{36u_m^2 |\mathcal{S}\Vert\mathcal{A}|\alpha^2}{\eta}M^2(q_k - q^*) + \frac{ 2|\mathcal{S}\Vert\mathcal{A}|\alpha^2}{\eta}M(q_k - q^*)(3\Vert q^*\Vert_\infty + r_{\max})^2\\
&\leq \frac{36u_m^2 |\mathcal{S}\Vert\mathcal{A}|\alpha^2}{\eta}M^2(q_k - q^*) + \alpha(1-\gamma)^2M^2(q_k - q^*) + \frac{ |\mathcal{S}|^2|\mathcal{A}|^2\alpha^3}{(1-\gamma)^2\eta^2}(3\Vert q^*\Vert_\infty + r_{\max})^4.
\end{aligned}
\end{equation*}

The term $T_4$ can be directly bounded as follow:
\begin{equation*}
\begin{aligned}
T_4 &= 3\alpha^2\langle \nabla M(q_k - q^*), \Bar{F}(q_k)\rangle^2\\
&\leq 3\alpha^2\left(\Vert\nabla M(q_k - q^*)\Vert_2\Vert\Bar{F}(q_k)\Vert_2\right)^2\\
&= 3\alpha^2\left(\Vert\nabla M(q_k - q^*) - \nabla M(q^* - q^*)\Vert_2\Vert\Bar{F}(q_k) - \Bar{F}(q^*)\Vert_2\right)^2\\
&\leq 3\alpha^2\left(\frac{\sqrt{|\mathcal{S}\Vert\mathcal{A}|}}{\eta}\|q_k-q^*\|_2\Vert\Bar{F}(q_k) - \Bar{F}(q^*)\Vert_\infty\right)^2\\
&\leq 3\alpha^2\left(\frac{2}{\eta}|\mathcal{S}\Vert\mathcal{A}|\Vert q_k - q^*\Vert_\infty^2\right)^2\\
&\leq \frac{12 u_m^4|\mathcal{S}|^2|\mathcal{A}|^2\alpha^2}{\eta^2}\Vert q_k - q^*\Vert_m^4 \\
&= \frac{48u_m^4|\mathcal{S}|^2|\mathcal{A}|^2\alpha^2}{\eta^2}M^2(q_k - q^*).\\
\end{aligned}
\end{equation*}

By Cauchy-Schwarz Inequality, we bound $T_5$ by the following three parts:

\begin{equation*}
\begin{aligned}
T_5 &\leq \underbrace{ 3 \langle \nabla M(q_k - q^*) - \nabla M(q_{k - t_{\alpha^2}} - q^*), F(x_k,q_k) - \Bar{F}(q_k)\rangle^2}_{T_{51}}\\
&+ \underbrace{ 3 \langle \nabla M(q_{k -t_{\alpha^2}} - q^*), F(x_k,q_k) - F(x_k,q_{k -t_{\alpha^2}})+\Bar{F}(q_{k - t_{\alpha^2}}) - \Bar{F}(q_k)\rangle^2}_{T_{52}}\\
&+ \underbrace{ 3 \langle \nabla M(q_{k - t_{\alpha^2}} - q^*), F(x_k,q_{k - t_{\alpha^2}})-\Bar{F}(q_{k - t_{\alpha^2}})\rangle^2}_{T_{53}}.
\end{aligned}
\end{equation*}

By Lemma A.3 in \cite{chen2021lyapunov}, for all $k \geq t_{\alpha^2}$ with $\alpha$ satisfying $\alpha t_{\alpha^2} \leq \frac{1}{12}$:

\[\begin{aligned}
T_{51} &\leq 3\left(\frac{144u_m^2|\mathcal{S}\Vert\mathcal{A}|\alpha t_{\alpha^2}}{\eta}M(q_k - q^*) + \frac{8|\mathcal{S}\Vert\mathcal{A}|\alpha t_{\alpha^2}}{\eta}(3\Vert q^*\Vert_\infty + r_{\max})^2\right)^2\\
&\leq \frac{124416u_m^4|\mathcal{S}|^2|\mathcal{A}|^2\alpha^2 t_{\alpha^2}^2}{\eta^2}M^2(q_k - q^*) + \frac{384|\mathcal{S}|^2|\mathcal{A}|^2\alpha^2 t_{\alpha^2}^2}{\eta^2}(3\Vert q^*\Vert_\infty + r_{\max})^4,\\
T_{52} &\leq 3\left(\frac{576u_m^2|\mathcal{S}\Vert\mathcal{A}|\alpha t_{\alpha^2}}{\eta}M(q_k - q^*) + \frac{32|\mathcal{S}\Vert\mathcal{A}|\alpha t_{\alpha^2}}{\eta}(3\Vert q^*\Vert_\infty + r_{\max})^2\right)^2\\
&\leq \frac{1990656u_m^4|\mathcal{S}|^2|\mathcal{A}|^2\alpha^2 t_{\alpha^2}^2}{\eta^2}M^2(q_k - q^*) + \frac{6144|\mathcal{S}|^2|\mathcal{A}|^2\alpha^2 t_{\alpha^2}^2}{\eta^2}(3\Vert q^*\Vert_\infty + r_{\max})^4.
\end{aligned}\]

For term $T_{53}$, we use the conditional expectation as follow:
\begin{equation}\label{T53}
\begin{aligned}
&\mathbb{E}[T_{53}|x_{k - t_{\alpha^2}}, q_{k - t_{\alpha^2}}] \\
=& 3 \mathbb{E}\left[\langle \nabla M(q_{k - t_{\alpha^2}} - q^*), F(x_k,q_{k - t_{\alpha^2}})-\Bar{F}(q_{k - t_{\alpha^2}})\rangle^2|x_{k - t_{\alpha^2}}, q_{k - t_{\alpha^2}}\right]\\
\end{aligned}
\end{equation}

Let $H = \nabla M(q_{k - t_{\alpha^2}} - q^*) \cdot \nabla M(q_{k - t_{\alpha^2}} - q^*)^{\top}$. Equation \eqref{T53} can be reformulated as follows:
\begin{equation*}
\begin{aligned}
&\quad\mathbb{E}[T_{53}|x_{k - t_{\alpha^2}},q_{k - t_{\alpha^2}}]\\  &= 3 \mathbb{E}\left[(F(x_k,q_{k - t_{\alpha^2}})-\Bar{F}(q_{k - t_{\alpha^2}}))^{\top} H (F(x_k,q_{k - t_{\alpha^2}})-\Bar{F}(q_{k - t_{\alpha^2}})) |x_{k - t_{\alpha^2}}, q_{k - t_{\alpha^2}} \right]\\
& = 3 \mathbb{E}\left[F(x_k,q_{k - t_{\alpha^2}})^{\top} H F(x_k,q_{k -t_{\alpha^2}}) - \Bar{F}(q_{k - t_{\alpha^2}})^{\top} H \Bar{F}(q_{k - t_{\alpha^2}})|x_{k - t_{\alpha^2}}, q_{k - t_{\alpha^2}} \right]\\
&\quad- 6 \mathbb{E}\left[(F(x_k,q_{k - t_{\alpha^2}})-\Bar{F}(q_{k - t_{\alpha^2}}))^{\top} H \Bar{F}(q_{k - t_{\alpha^2}})|x_{k - t_{\alpha^2}}, q_{k - t_{\alpha^2}} \right]\\
& = 3\left(\sum_{x \in \mathcal{X}}\left(P^{t_{\alpha^2}}\left(x_{k-t_{\alpha^2}}, x\right)-\mu_{\mathcal{X}}(x)\right) F\left(x, q_{k - t_{\alpha^2}}\right)^{\top} H F\left(x, q_{k - t_{\alpha^2}}\right)\right)\\
& \quad- 6\left(\sum_{x \in \mathcal{X}}\left(P^{t_{\alpha^2}}\left(x_{k-t_{\alpha^2}}, x\right)-\mu_{\mathcal{X}}(x)\right) F\left(x, q_{k - t_{\alpha^2}}\right)^{\top} H \Bar{F}\left(q_{k - t_{\alpha^2}}\right)\right)\\
& \stackrel{(a)}{\leq} 6\alpha^2 \Vert F(\widetilde{x_0},q_{k - t_{\alpha^2}})\Vert_2^2 \Vert H\Vert_2 + 12 \alpha^2 \Vert F(\widetilde{x_1},q_{k - t_{\alpha^2}})\Vert_2 \Vert H\Vert_2 \Vert\bar{F}(q_{k - t_{\alpha^2}})\Vert_2\\
& \leq \frac{18\alpha^2 |\mathcal{S}|^2|\mathcal{A}|^2}{\eta^2} (2\Vert q_{k - t_{\alpha^2}}\Vert_\infty + r_{\max})^2 \Vert q_{k-t_{\alpha^2}} - q^*\Vert_\infty^2\\
& \leq \frac{18\alpha^2 |\mathcal{S}|^2|\mathcal{A}|^2}{\eta^2}(2\Vert q_{k-t_{\alpha^2}} - q^*\Vert_\infty + 2\Vert q^*\Vert_\infty + r_{\max})^2\Vert q_{k-t_{\alpha^2}} - q^*\Vert_\infty^2\\
& \leq \frac{18\alpha^2 |\mathcal{S}|^2|\mathcal{A}|^2}{\eta^2}(2\Vert q_{k-t_{\alpha^2}} - q_k\Vert_\infty+2\Vert q_k - q^*\Vert_\infty + 2\Vert q^*\Vert_\infty + r_{\max})^2   \cdot (\Vert q_{k-t_{\alpha^2}} - q_k\Vert_\infty + \Vert q_k - q^*\Vert_\infty)^2\\
& \leq \frac{18\alpha^2 |\mathcal{S}|^2|\mathcal{A}|^2}{\eta^2}(2(\Vert q_k\Vert_\infty + \frac{r_{\max}}{3})+2\Vert q_k - q^*\Vert_\infty + 2\Vert q^*\Vert_\infty + r_{\max})^2 \cdot ((\Vert q_k\Vert_\infty + \frac{r_{\max}}{3}) + \Vert q_k - q^*\Vert_\infty)^2\\
& \leq \frac{18\alpha^2 |\mathcal{S}|^2|\mathcal{A}|^2}{\eta^2}(6\Vert q_k - q^*\Vert_\infty + 6\Vert q^*\Vert_\infty + 2r_{\max})^2(3\Vert q_k - q^*\Vert_\infty + 3\Vert q^*\Vert_\infty + r_{\max})^2\\
& = \frac{72\alpha^2 |\mathcal{S}|^2|\mathcal{A}|^2}{\eta^2}(3\Vert q_k - q^*\Vert_\infty + 3\Vert q^*\Vert_\infty + r_{\max})^4\\
& \leq \frac{186624\alpha^2 |\mathcal{S}|^2|\mathcal{A}|^2}{\eta^2}M^2(q_k - q^*) + \frac{576\alpha^2 |\mathcal{S}|^2|\mathcal{A}|^2}{\eta^2}(3\Vert q^*\Vert_\infty + r_{\max
})^4.
\end{aligned}
\end{equation*}
where (a) follows with some $\widetilde{x_0
}, \widetilde{x_1} \in \mathcal{X}$. 
Here we use the facts that $\sum_{x \in \mathcal{X}} |P^{t_{\alpha^2}}\left(x_{k-t_{\alpha^2}}, x\right)-\mu_{\mathcal{X}}(x)| \leq 2\alpha^2$ (by Definition~\ref{def:mixing_time} of mixing time) and $\Vert q_{k-t_{\alpha^2}} - q_k\Vert_\infty \leq \Vert q_k\Vert_\infty + \frac{r_{\max}}{3}$, which has been proved in Lemma A.2 in \cite{chen2021lyapunov}.

By putting these three terms together, we  obtain the following bound for $\mathbb{E}[T_5]$:
\begin{equation*}
\begin{aligned}
\mathbb{E}(T_5)&\leq \frac{|\mathcal{S}|^2|\mathcal{A}|^2 (2115072 u_m^4\alpha^2 t_{\alpha^2}^2 + 186624\alpha^2)}{\eta^2}M^2(q_k - q^*) \\
& \quad+\frac{|\mathcal{S}|^2|\mathcal{A}|^2 (6528\alpha^2 t_{\alpha^2}^2 + 576\alpha^2)}{\eta^2}(3\Vert q^*\Vert_\infty + r_{\max})^4.
\end{aligned}
\end{equation*}

By Lemma A.5 in \cite{chen2021lyapunov}, we have
\begin{equation*}
\begin{aligned}
T_6 &\leq \frac{3\alpha^4}{4\eta^2}\left(36u_m^2 |\mathcal{S}\Vert\mathcal{A}|M(q_k - q^*) + 2|\mathcal{S}\Vert\mathcal{A}|(3\Vert q^*\Vert_\infty + r_{\max})^2\right)^2\\
& \leq \frac{1944u_m^4 |\mathcal{S}|^2|\mathcal{A}|^2 \alpha^4 }{\eta^2}M^2(q_k - q^*)  + \frac{6|\mathcal{S}|^2|\mathcal{A}|^2\alpha^4}{\eta^2}(3\Vert q^*\Vert_\infty + r_{\max})^4.
\end{aligned}
\end{equation*}

Using the above bounds for $T_1 - T_6$, we can finally bound $\mathbb{E}[M^2(q_{k+1} - q^*)]$ by following:

\begin{equation*}
\begin{aligned}
\mathbb{E}[M^2(q_{k+1} - q^*)] \leq & (1 -3\alpha (1 - \gamma)^{2})\mathbb{E}[M^2(q_{k} - q^*)] \\
&+ \frac{36u_m^2 |\mathcal{S}\Vert\mathcal{A}|\alpha^2}{\eta}\E[M^2(q_k - q^*)] + \alpha(1-\gamma)^2\E[M^2(q_k - q^*)] \\
&+ \frac{ |\mathcal{S}|^2|\mathcal{A}|^2\alpha^3}{(1-\gamma)^2\eta^2}(3\Vert q^*\Vert_\infty + r_{\max})^2 + \frac{48u_m^4|\mathcal{S}|^2|\mathcal{A}|^2\alpha^2}{\eta^2}\mathbb{E}[M^2(q_{k} - q^*)]\\
& + \frac{|\mathcal{S}|^2|\mathcal{A}|^2 (4230144 u_m^4\alpha^3 t_{\alpha^2}^2 + 373248\alpha^3)}{\eta^2(1-\gamma)^2}\mathbb{E}[M^2(q_{k} - q^*)] \\
&+ \frac{|\mathcal{S}|^2|\mathcal{A}|^2 (13056\alpha^3 t_{\alpha^2}^2 + 1152\alpha^3)}{\eta^2(1-\gamma)^2}(3\Vert q^*\Vert_\infty + r_{\max})^4\\
&+\frac{1944u_m^4 |\mathcal{S}|^2|\mathcal{A}|^2 \alpha^4 }{\eta^2}\mathbb{E}[M^2(q_{k} - q^*)]  + \frac{6|\mathcal{S}|^2|\mathcal{A}|^2\alpha^4}{\eta^2}(3\Vert q^*\Vert_\infty + r_{\max})^4\\
\leq & (1 -\alpha (1 - \gamma)^{2})\mathbb{E}[M^2(q_{k} - q^*)] \\
&+ \frac{ |\mathcal{S}|^2|\mathcal{A}|^2(374007\alpha^3 + 13056\alpha^3 t_{\alpha^2}^2)}{\eta^2(1 - \gamma)^2}(3\Vert q^*\Vert_\infty + r_{\max})^4,
\end{aligned}
\end{equation*}
where there exists a $\alpha_0 > 0 $ such that the last step always hold for all $\alpha \leq \alpha_0.$

Then, we obtain that for all $k \geq t_{\alpha^2}:$
\begin{equation*}
\begin{aligned}
\mathbb{E}[M^2(q_{k} - q^*)] &\leq \mathbb{E}[M^2(q_{t_{\alpha^2}} - q^*)] (1 -\alpha (1 - \gamma)^{2})^{k - t_{\alpha^2}} \\
&+  \frac{ |\mathcal{S}|^2|\mathcal{A}|^2(374007\alpha^2 + 13056\alpha^2 t_{\alpha^2}^2)}{\eta^2(1-\gamma)^4}(3\Vert q^*\Vert_\infty + r_{\max})^4.
\end{aligned}
\end{equation*}

We can choose $\eta = \frac{(1-\gamma)^2}{\sqrt{|\mathcal{S}\Vert\mathcal{A}|} }$ satisfying equation \eqref{eta} and by \cite[Theorem A.1]{chen2021lyapunov}, we obtain the following bound for $\E[\Vert q_k - q^*\Vert_\infty^4]$:
\begin{equation*}
\begin{aligned}
\mathbb{E}[\Vert q_k - q^*\Vert_\infty^4] &\leq 4 u_m^4 \mathbb{E}[M^2(q_{k} - q^*)]\\
&\leq 4 u_m^4 \mathbb{E}[M^2(q_{t_{\alpha^2}} - q^*)](1 -\alpha (1 - \gamma)^{2})^{k - t_{\alpha^2}}\\
&+ 4 u_m^4  \frac{ |\mathcal{S}|^3|\mathcal{A}|^3(374007\alpha^2 + 13056\alpha^2 t_{\alpha^2}^2)}{(1-\gamma)^8}(3\Vert q^*\Vert_\infty + r_{\max})^4\\
&\leq \frac{u_m^4}{l_m^4}\mathbb{E}[\Vert q_{t_{\alpha^2}} - q^*\Vert_\infty^4](1 -\alpha (1 - \gamma)^{2})^{k - t_{\alpha^2}} \\
&+ \frac{ 4u_m^4|\mathcal{S}|^3|\mathcal{A}|^3(374007\alpha^2 + 13056\alpha^2 t_{\alpha^2}^2)}{(1-\gamma)^8}(3\Vert q^*\Vert_\infty + r_{\max})^4\\
&\leq \frac{u_m^4}{l_m^4}\mathbb{E}((\Vert q_{t_{\alpha^2}} - q_0\Vert_\infty + \Vert q_0-q^*\Vert_\infty)^4)(1 -\alpha (1 - \gamma)^{2})^{k - t_{\alpha^2}} \\
&+ \frac{ 4u_m^4|\mathcal{S}|^3|\mathcal{A}|^3(374007\alpha^2 + 13056\alpha^2 t_{\alpha^2}^2)}{(1-\gamma)^8}(3\Vert q^*\Vert_\infty + r_{\max})^4\\
&\leq \frac{u_m^4}{l_m^4}(\Vert q_0\Vert_\infty + \Vert q_0-q^*\Vert_\infty + \frac{r_{\max}}{3})^4(1 -\alpha (1 - \gamma)^{2})^{k - t_{\alpha^2}} \\
&+ \frac{ 4u_m^4|\mathcal{S}|^3|\mathcal{A}|^3(374007\alpha^2 + 13056\alpha^2 t_{\alpha^2}^2)}{(1-\gamma)^8}(3\Vert q^*\Vert_\infty + r_{\max})^4.
\end{aligned}
\end{equation*}

By  Lemma \ref{lemma:M}, we can let $l_m = (1 + \frac{\eta}{|\mathcal{S}\Vert\mathcal{A}|})^{1/2}$, $u_m = (1 + \eta)^{1/2}$. Define
\begin{equation*}
\begin{aligned}
b_1 &= \frac{(1 +  \frac{(1-\gamma)^2}{\sqrt{|\mathcal{S}\Vert\mathcal{A}|} })^2}{(1 + \frac{(1-\gamma)^2}{(|\mathcal{S}\Vert\mathcal{A}|)^{\frac{3}{2}}})^2}(\Vert q_0\Vert_\infty + \Vert q_0-q^*\Vert_\infty + \frac{r_{\max}}{3})^4,\\
b_2 &= \frac{ 374007\times4(1 + \frac{(1-\gamma)^2}{\sqrt{(|\mathcal{S}\Vert\mathcal{A}|)}})^2|\mathcal{S}|^3|\mathcal{A}|^3}{(1-\gamma)^8}(3\Vert q^*\Vert_\infty + r_{\max})^4,\\
b_3 &= \frac{ 13056\times4(1 + \frac{(1-\gamma)^2}{\sqrt{(|\mathcal{S}\Vert\mathcal{A}|)}})^2|\mathcal{S}|^3|\mathcal{A}|^3}{(1-\gamma)^8}(3\Vert q^*\Vert_\infty + r_{\max})^4.
\end{aligned}
\end{equation*}
We have for all $k \geq t_{\alpha^2}$,
\begin{equation*}
\mathbb{E}[\Vert q_k - q^*\Vert_\infty^4] \leq b_1(1 -\alpha (1 - \gamma)^{2})^{k - t_{\alpha^2}} +b_2\alpha^2 + b_3\alpha^2 t_{\alpha^2}^2.
\end{equation*}

This completes the proof of Proposition~\ref{44tabular}. 
\end{proof}

\subsection{Step 2: Basic Adjoint Relationship}

We first derive a recursive relationship for the following quantities
\[z(i) := \mathbb{E}\left[q_\infty \mid x_\infty = i\right], ~~~ i\in \mathcal{X}.
\]

Recall that $(x_k)_{k \geq 0}$ is a time-homogeneous Markov chain with transition probability matrix $P = (p_{ij})$ and a unique
stationary distribution $\mu_{\mathcal{X}}$. Theorem \ref{limit4tabular} shows that $(x_k, q_k)_{k \geq 0}$ converges in distribution to a limit $(x_{\infty},q_{\infty}) \sim \Bar{\mu}$, with marginal $q_\infty \sim \mu$ and $x_\infty \sim \mu_{\mathcal{X}}$. Given $(x_\infty, q_\infty)$, let $x_{\infty+1}$ be a random variable with conditional distribution $\mathbb{P}(x_{\infty+1} = j \mid x_{\infty} = i) = p_{ij}$, and $q_{\infty+1} = q_{\infty} + \alpha F(x_\infty, q_\infty)  $.

Since $(x_\infty, q_\infty)$ is in the stationary, $(x_{\infty+1}, q_{\infty+1})$ also follows the stationary distribution $\Bar{\mu}$. Let $d = |\mathcal{S}\Vert\mathcal{A}|$. Therefore, for any test function $f: \mathcal{X} \times \mathbb{R}^d \mapsto \mathbb{R}^d$ that satisfies $\|f(x,q)\|_\infty \leq C(1 + \|q\|_\infty^2)$ for some $C\in \R$, the following relationship holds \cite[Theorem 6.9]{villani2009optimal}
\begin{equation*}
\mathbb{E}[f(x_\infty, q_\infty)] = \mathbb{E}[f(x_{\infty+1}, q_{\infty+1})],
\end{equation*}
which is called Basic Adjoint Relationship (BAR).

Consider the test function $f^{(i)}$, $i \in \mathcal{X}$, defined as
\[f^{(i)}(x,q) = q \cdot \mathds{1}\{x = i\}.\]
Substituting $f = f^{(i)}$ into BAR gives
\begin{equation}\label{bar}
\mathbb{E}[q_\infty \cdot \mathds{1}\{x_\infty = i\}] = \mathbb{E}[q_{\infty+1} \cdot \mathds{1}\{x_{\infty+1} = i\}].
\end{equation}
To simplify the presentation, we denote by $\nu(i):= \mu_{\mathcal{X}}(i)$ the probability of the Markov chain $(x_k)_{k\geq 0}$ being in state $i \in \mathcal{X}$ when in stationary. The LHS of equation \eqref{bar} can be written as follows
$$
\begin{aligned}
\mathbb{E}\left[q_\infty \cdot \mathds{1}\{x_\infty = i\}\right] & = \nu(i) \cdot \mathbb{E}\left[q_\infty \mid x_\infty = i\right] = \nu(i) z(i).
\end{aligned}$$

Recall that $\hat{P} = (\hat{p}_{ij})$ is the transition kernel of the time-reversal of the Markov chain $(x_k)_{k \geq 0}$. The RHS of equation \eqref{bar} can be reformulated as
\begin{equation*}
\begin{aligned}
\mathbb{E}\left[q_{\infty+1} \cdot \mathds{1}\{x_{\infty+1} = i\}\right]& = \nu(i) \mathbb{E}\left[q_{\infty+1} \mid x_{\infty+1} = i\right]\\
& = \nu(i) \mathbb{E}\left[q_{\infty} + \alpha F(x_\infty, q_\infty)  \mid x_{\infty+1} = i \right]\\
& = \nu(i) \sum_{j \in \mathcal{X}} \hat{p}_{ij}\mathbb{E}\left[q_{\infty} + \alpha F(x_\infty, q_\infty) \mid  x_{\infty} = j, x_{\infty+1} = i \right]\\
& = \nu(i) \sum_{j \in \mathcal{X}}\hat{p}_{ij} \mathbb{E}\left[q_{\infty} + \alpha F(j, q_\infty)  \mid  x_{\infty} = j \right].
\end{aligned}    
\end{equation*}
The last step follows from the fact that condition on $x_k,$ $q_k$ is conditionally independent of $x_{k+1}$ for all $k\geq 1.$

By Proposition \ref{gradient}, we can further rewrite the above equation as
\begin{equation*}
\begin{aligned}
& \quad \mathbb{E}\left[q_{\infty+1} \cdot \mathds{1}\{x_{\infty+1} = i\}\right]\\ & = \nu(i) \sum_{j \in \mathcal{X}} \hat{p}_{ij}\mathbb{E}\left[q_{\infty} + \alpha \left(F(j, q^*)  +(G_{q^*}(j) - I_d)(q_\infty - q^*) + R(j,q_\infty) \right) \mid  x_{\infty} = j \right]\\
& = \nu(i) \sum_{j \in \mathcal{X}} \hat{p}_{ij}\left[z(j) + \alpha\left(F(j,q^*) +(G_{q^*}(j) - I_d)(z(j) - q^*) + \mathbb{E}(R(j,q_\infty)\mid x_\infty = j)\right)\right].
\end{aligned}
\end{equation*}

We thus obtain the following recursive relationship for $\{z(i)\}_{i\in \mathcal{X}}:$
\begin{equation}\label{eq:z_recursion}
\begin{aligned}
z(i) &= \sum_{j \in \mathcal{X}} \hat{p}_{ij}\left[z(j) + \alpha\left(F(j,q^*) +(G_{q^*}(j) - I_d)(z(j) - q^*) + \mathbb{E}(R(j,q_\infty)\mid x_\infty = j)\right)\right]\\
& = \sum_{j \in \mathcal{X}} \hat{p}_{ij}\left[z(j) + \alpha\left(F(j,q^*) +(G_{q^*}(j) - I_d)(z(j) - q^*)\right)\right] + \alpha\mathbb{E}\left[R(x_\infty, q_\infty) \mid x_{\infty + 1} = i\right].
\end{aligned}
\end{equation}

Note that the second term of the RHS of equation \eqref{eq:z_recursion} can be bounded as
$$
\begin{aligned}
\mathbb{E}\left[R(x_\infty, q_\infty) \mid x_{\infty + 1} = i\right] & = \frac{1}{\nu(i)}\mathbb{E}\left[R(x_\infty, q_\infty) \mathds{1} \{x_{\infty + 1}= i\} \right]\\
& \overset{(i)}{\leq} \frac{1}{\nu(i)} \mathbb{E}\left[R(x_\infty, q_\infty) \right]\\
& \overset{(ii)}{=} \mathcal{O}(\alpha^2 + \alpha^2t_{\alpha^2}^2),
\end{aligned}$$
where (i) holds because $R(x, q)$ is always positive, as shown in the proof of Proposition \ref{gradient}; (ii) follows from Proposition \ref{gradient}, Proposition \ref{44tabular} and Hölder's inequality.

Let $A(x) = G_{q^*}(x) - I_d$ and $b(x) = F(x,q^*) - (G_{q^*}(x) - I_d)q^*$. Let $D$ denote the operator given by $(D f)(x)=A(x) f(x)$ for each $x \in \mathcal{X}$. We thus can simplify equation \eqref{eq:z_recursion} by 

\begin{equation}\label{eq:z}
z = \hat{P}\left(z + \alpha(Dz + b)\right) + \mathcal{O}(\alpha^3 + \alpha^3t_{\alpha^2}^2).
\end{equation}

\subsection{Step 3: Setting up System of $\delta$}

Define the difference 
$$\delta(i) := z(i) - \mu_{\mathcal
X}z \text{ for each } i \in \mathcal{X},$$
where $\mu_{\mathcal
X}z:=\sum_{i \in \mathcal{X}}\nu(i)z(i)$. 
Let $\Pi=1 \otimes \mu_{\mathcal{X}} $. Then, by applying the operator $(\hat{P} - \Pi)$ to both side of above equation we obtain

$$
\left(P^*-\Pi\right) z=\left(P^*-\Pi\right) \delta.
$$
Subtracting $\Pi z$ from both sides of equation \eqref{eq:z}, we obtain
\begin{equation}\label{eq:delta_recursion}
\begin{aligned}
\delta &= (\hat{P} - \Pi)z + \alpha\hat{P}(Dz+b) + \mathcal{O}(\alpha^3 + \alpha^3t_{\alpha^2}^2)\\
&=(\hat{P} - \Pi)\delta + \alpha\hat{P}(Dz+b) + \mathcal{O}(\alpha^3 + \alpha^3t_{\alpha^2}^2).
\end{aligned}
\end{equation}

Applying $\mu_{\mathcal{X}}$ to both sides of equation \eqref{eq:z}, we obtain

\begin{equation}\label{eq:0}
\mu_{\mathcal{X}}(Dz + b) = \mathcal{O}(\alpha^2 + \alpha^2t_{\alpha^2}^2).
\end{equation}

Subtracting equation \eqref{eq:0} from equation \eqref{eq:delta_recursion}, we obtain

\[\delta = (\hat{P} - \Pi)\delta + \alpha(\hat{P} - \Pi)(Dz+b) + \mathcal{O}(\alpha^3 + \alpha^3t_{\alpha^2}^2).\]

Then, we have

\[(I-\hat{P} + \Pi)\delta = \alpha(\hat{P} - \Pi)(Dz+b) + \mathcal{O}(\alpha^3 + \alpha^3t_{\alpha^2}^2).\]

It is well-known that $(I-\hat{P} + \Pi)^{-1}$ exist. Therefore, we obtain

\begin{equation}\label{eq:z2delta}\delta = \alpha(I-\hat{P} + \Pi)^{-1}(\hat{P} - \Pi)(Dz+b) + \mathcal{O}(\alpha^3 + \alpha^3t_{\alpha^2}^2).\end{equation}

\subsection{Step 4: Establishing $\delta = \mathcal{O}(\alpha)$}

In this sub-section, we show that $\vec{\delta}=\mathcal{O}(\alpha)$, as stated in the following Lemma. 

\begin{lem}\label{delta}
Under Assumption \ref{MC}, and $\alpha t_\alpha \leq c_0 \frac{(1-\beta)^2}{\log(|\mathcal{S}\Vert\mathcal{A}|)}$, we have
$$
\|\vec{\delta}\|_\infty \leq \alpha \cdot B^{\prime \prime}(r, \gamma, P)
$$
for some number $B^{\prime \prime}(r, \gamma, P)\in \R$  that is independent of $\alpha$.
\end{lem}

\begin{proof}[Proof of Lemma \ref{delta}]
Recalling the definition for $z(i)$, we have
$$
z(i)=\mathbb{E}\left[q_{\infty} \mid x_{\infty}=i\right]=\frac{\mathbb{E}\left[q_{\infty} \mathds{1}\left\{x_{\infty}=i\right\}\right]}{\nu(i)}.
$$

Then by Lemma \ref{variance} and the fact that $\nu(i) > 0$, we have

$$
\begin{aligned}
\left\|z(i)\right\|_\infty & \leq \frac{\mathbb{E}\left[\left\|q_{\infty}\right\|_\infty\right]}{\nu(i)}\leq \frac{1}{\nu_{\min }} \cdot \sqrt{2c_{Q}c_{0} + 2\Vert q^*\Vert_\infty^2},
\end{aligned}
$$
where $\nu_{\min}:=\min \limits_{i} \nu(i)>0$.

By equation \eqref{eq:z2delta}, we conclude that
$$
\|\vec{\delta}\|_\infty \leq \alpha \cdot B^{\prime \prime}(r, \gamma, P)
$$
for some number $B^{\prime \prime}(r, \gamma, P)$  that is independent of $\alpha$.

\end{proof}

\subsection{Step 5: Expansion of the bias}

By definition, $\bar{F}(q^*) = 0$ and $R(x,q^*) \equiv 0$. Define $\bar{A} = \mathbb{E}_{\mu_{\mathcal{X}}}A(x)$ and $\bar{b} = \mathbb{E}_{\mu_{\mathcal{X}}}b(x)$. Then, we have $\bar{A}q^* + \bar{b} = 0$. From Proposition \ref{gradient}, $\bar{A}$ is a non-singular matrix. Define $\bar{D}$ be the normalized $D$ such that $(\bar{D} f)(x)=\bar{A}^{-1}A(x) f(x)$. Therefore, we obtain

\[
\begin{aligned}
q^* &= -\bar{A}^{-1}\bar{b}\\
&= -\bar{A}^{-1}\mu_{\mathcal{X}}b\\
&= \mu_{\mathcal{X}}\bar{D}z + \mathcal{O}(\alpha^2 + \alpha^2t_{\alpha^2}^2),\\
\end{aligned}\]
where the last inequality holds by equation \eqref{eq:0}.

Because $\delta = z- \Pi z$, we can further obtain

\[q^* = \mu_{\mathcal{X}}\bar{D}\delta + \mu_{\mathcal{X}}z + \mathcal{O}(\alpha^2 + \alpha^2t_{\alpha^2}^2).\]

Then,

\[\mu_{\mathcal{X}}z = q^* - \mu_{\mathcal{X}}\bar{D}\delta + \mathcal{O}(\alpha^2 + \alpha^2t_{\alpha^2}^2).\]

\[z(i) = \delta(i) + \mu_{\mathcal{X}}z = \delta(i)+q^* - \mu_{\mathcal{X}}\bar{D}\delta + \mathcal{O}(\alpha^2 + \alpha^2t_{\alpha^2}^2).\]

Therefore, we obtain

\begin{equation}\label{eq:delta2z}
z = q^* + (I-\Pi\bar{D})\delta + \mathcal{O}(\alpha^2 + \alpha^2t_{\alpha^2}^2).
\end{equation}

Substituting equation \eqref{eq:delta2z} into equation \eqref{eq:z2delta}, we obtain
\[
\begin{aligned}
\delta &= \alpha(I-\hat{P} + \Pi)^{-1}(\hat{P} - \Pi)(Dz+b) + \mathcal{O}(\alpha^3 + \alpha^3t_{\alpha^2}^2)\\
&=\alpha\underbrace{(I-\hat{P} + \Pi)^{-1}(\hat{P} - \Pi)(Aq^*+b)}_{v}\\
&\qquad+ \alpha\underbrace{(I-\hat{P} + \Pi)^{-1}(\hat{P} - \Pi)D(I-\Pi \bar{D})}_{\Xi}\delta\\
&\qquad+  \mathcal{O}(\alpha^3 + \alpha^3t_{\alpha^2}^2)\\
& = \alpha v + \alpha \Xi \delta + \mathcal{O}(\alpha^3 + \alpha^3t_{\alpha^2}^2).
\end{aligned}
\]

Therefore, we can finally obtain

\[
\begin{aligned}
\mathbb{E}(q_\infty)  &= \mu_{\mathcal{X}}z \\
&= q^* - \mu_{\mathcal{X}}\bar{D}\delta + \mathcal{O}(\alpha^2 + \alpha^2t_{\alpha^2}^2)\\
& = q^* - \alpha \mu_{\mathcal{X}}\bar{D}v - \alpha\mu_{\mathcal{X}}\bar{D}\Xi \delta  + \mathcal{O}(\alpha^2 + \alpha^2t_{\alpha^2}^2)
\end{aligned}
\]

Let $B = -\mu_{\mathcal{X}}\bar{D}v$. By Lemma \ref{delta}, we have $\mu_{\mathcal{X}}\bar{D}\Xi \delta = \mathcal{O}(\alpha )$

Therefore, we have

\[\mathbb{E}(q_\infty) = q^* + \alpha B + \mathcal{O}(\alpha^2 + \alpha^2t_{\alpha^2}^2)\]

with

\begin{equation}\label{eq:B}
B = -\mu_{\mathcal{X}}\bar{D}(I-\hat{P} + \Pi)^{-1}(\hat{P} - \Pi)(Aq^*+b) .
\end{equation}

We complete the proof of Theorem \ref{thm:bias}.

\section{Proof of Corollary \ref{co:pr}} \label{sec:proof_coro_moment}

In this section, we provide the proof of the first and second moment bounds in Corollary \ref{co:pr}.

\subsection{First Moment}
First, we have
$$
\mathbb{E}\left[\bar{q}_{k_0, k}\right]-q^*=\left(\mathbb{E}\left[q_{\infty}\right]-q^*\right)+\frac{1}{k-k_0} \underbrace{\sum_{t=k_0}^{k-1} \mathbb{E}\left[q_t-q_{\infty}\right]}_{T_1}.
$$
By Corollary \ref{co4limit}, we have that for $k \geq t_\alpha$,
\[\Vert\mathbb{E}[q_k] - \mathbb{E}[q_\infty]\Vert_{\infty}  \leq C(r,\gamma,P)\cdot \left(1-\frac{\left(1-\beta\right) \alpha}{2}\right)^{\frac{k-t_\alpha}{2}}.\]

Then, when $\alpha t_\alpha \leq 1$, we have the following bound for $T_1$,
\begin{equation*}
\begin{aligned}
\left\|T_1\right\|_\infty & =\left\|\sum_{t=k_0}^{k-1} \mathbb{E}\left[q_t-q_{\infty}\right]\right\|_\infty \leq \sum_{t=k_0}^{k-1}\left\|\mathbb{E}\left[q_t\right]-\mathbb{E}\left[q_{\infty}\right]\right\|_\infty \\
& \leq C(r,\gamma,P)\left(1-\frac{\left(1-\beta\right) \alpha}{2}\right)^{\frac{k_0-t_\alpha}{2}}\frac{1}{1 - \sqrt{1 - \frac{\left(1-\beta\right) \alpha}{2}}}\\
& \leq C(r,\gamma,P)\left(1-\frac{\left(1-\beta\right) \alpha}{2}\right)^{\frac{k_0-t_\alpha}{2}}\frac{4}{\left(1-\beta\right) \alpha} \\
& \stackrel{(i)}{\leq} C(r,\gamma,P) \exp\left(-\frac{(1-\beta)\alpha(k_0 - t_\alpha)}{4}\right)\frac{4}{\left(1-\beta\right) \alpha}\\
& \leq C^{\prime\prime}(r,\gamma,P) \cdot \frac{1}{\alpha} \cdot \exp\left(-\frac{\alpha(1-\beta)k_0}{4}\right),
\end{aligned}
\end{equation*}
where $(i)$ follows from the inequality that $(1-u)^m\leq \exp(-um)$ for $0<u<1.$

Together with Theorem \ref{thm:bias}, we have
$$
\mathbb{E}\left[\bar{q}_{k_0, k}\right]-q^*= \alpha B(r, \gamma, P) + \mathcal{O}(\alpha^2 + \alpha^2t_{\alpha^2}^2)+\mathcal{O}\left(\frac{1}{\alpha(k - k_0)} \exp\left(-\frac{\alpha(1-\beta)k_0}{4}\right)\right),
$$
thereby finishing the proof of equation \eqref{eq:PR_first} for the first moment.

\subsection{Second Moment}

We first derive the bound for the second moment of the tail-averaged iterate. Note that
$$
\begin{aligned}
&\quad \mathbb{E}  {\left[\left(\bar{q}_{k_0, k}-\mathbb{E}\left[q_{\infty}\right]\right)\left(\bar{q}_{k_0, k}-\mathbb{E}\left[q_{\infty}\right]\right)^{\top}\right] } \\
= & \frac{1}{\left(k-k_0\right)^2} \mathbb{E}\left[\left(\sum_{t=k_0}^{k-1} (q_{t}-\mathbb{E}\left[q_{\infty}\right])\right)\left(\sum_{t=k_0}^{k-1} (q_{t}-\mathbb{E}\left[q_{\infty}\right])\right)^{\top}\right] \\
= & \underbrace{\frac{1}{\left(k-k_0\right)^2} \sum_{t=k_0}^{k-1} \mathbb{E}\left[\left(q_t-\mathbb{E}\left[q_{\infty}\right]\right)\left(q_t-\mathbb{E}\left[q_{\infty}\right]\right)^{\top}\right]}_{T_1} \\
& +\underbrace{\frac{1}{\left(k-k_0\right)^2} \sum_{t=k_0}^{k-1} \sum_{l=t+1}^{k-1}\left(\mathbb{E}\left[\left(q_t-\mathbb{E}\left[q_{\infty}\right]\right)\left(q_l-\mathbb{E}\left[q_{\infty}\right]\right)^{\top}\right]+\mathbb{E}\left[\left(q_l-\mathbb{E}\left[q_{\infty}\right]\right)\left(q_t-\mathbb{E}\left[q_{\infty}\right]\right)^{\top}\right]\right)}_{T_2} .
\end{aligned}
$$

For the term $T_1$, we have the following decomposition,
\begin{equation}\label{t1}
\begin{aligned}
&\quad \mathbb{E}\left[\left(q_t-\mathbb{E}\left[q_{\infty}\right]\right)\left(q_t-\mathbb{E}\left[q_{\infty}\right]\right)^{\top}\right] \\
& =\mathbb{E}\left[q_t q_t^{\top}-q_t \mathbb{E}\left[q_{\infty}^{\top}\right]-\mathbb{E}\left[q_{\infty}\right] q_t^{\top}+\mathbb{E}\left[q_{\infty}\right] \mathbb{E}\left[q_{\infty}^{\top}\right]\right] \\
& =\mathbb{E}\left[q_t q_t^{\top}\right]-\mathbb{E}\left[q_t\right] \mathbb{E}\left[q_{\infty}^{\top}\right]-\mathbb{E}\left[q_{\infty}\right] \mathbb{E}\left[q_t^{\top}\right]+\mathbb{E}\left[q_{\infty}\right] \mathbb{E}\left[q_{\infty}^{\top}\right] \\
& =\left(\mathbb{E}\left[q_t q_t^{\top}\right]-\mathbb{E}\left[q_{\infty} q_{\infty}^{\top}\right]\right)+\left(\mathbb{E}\left[q_{\infty} q_{\infty}^{\top}\right]-\mathbb{E}\left[q_{\infty}\right] \mathbb{E}\left[q_{\infty}^{\top}\right]\right)\\
&\quad -\left(\mathbb{E}\left[q_t\right] \mathbb{E}\left[q_{\infty}^{\top}\right]+\mathbb{E}\left[q_{\infty}\right] \mathbb{E}\left[q_t^{\top}\right]-2 \mathbb{E}\left[q_{\infty}\right] \mathbb{E}\left[q_{\infty}^{\top}\right]\right) \\
& =\left(\mathbb{E}\left[q_t q_t^{\top}\right]-\mathbb{E}\left[q_{\infty} q_{\infty}^{\top}\right]\right)+\operatorname{Var}\left(q_{\infty}\right)-\mathbb{E}\left[q_t-q_{\infty}\right] \mathbb{E}\left[q_{\infty}^{\top}\right]-\mathbb{E}\left[q_{\infty}\right] \mathbb{E}\left[\left(q_t-q_{\infty}\right)^{\top}\right]
\end{aligned}
\end{equation}

Corollary \ref{co4limit} and Lemma \ref{variance} imply the following bounds for $k \geq t_\alpha$,
\begin{align}
\mathbb{E}\left[\left\|q_t-q_{\infty}\right\|_\infty\right] & \leq C(r,\gamma,P)\cdot \left(1-\frac{\left(1-\beta\right) \alpha}{2}\right)^{\frac{t-t_\alpha}{2}} \label{eq:exp_qt}\\
\left\|\mathbb{E}\left[q_t q_t^{\top}\right]-\mathbb{E}\left[q_{\infty} q_{\infty}^{\top}\right]\right\|_\infty & \leq C^{\prime}(r,\gamma,P)\cdot \left(1-\frac{\left(1-\beta\right) \alpha}{2}\right)^{\frac{t-t_\alpha}{2}} \nonumber\\
\mathbb{E}\left[\left\|q_{\infty}\right\|_\infty\right] & \leq C^{\prime \prime}(r,\gamma,P), \nonumber\\
\operatorname{Var}\left(q_{\infty}\right) & \leq C^{\prime \prime \prime}(r,\gamma,P) \cdot \alpha t_\alpha. \label{eq:var_qinf}
\end{align}
Substituting these bounds into equation \eqref{t1}, we have
\[\mathbb{E}\left[\left(q_t-\mathbb{E}\left[q_{\infty}\right]\right)\left(q_t-\mathbb{E}\left[q_{\infty}\right]\right)^{\top}\right] = \mathcal{O}\left(\left(1-\frac{\left(1-\beta\right) \alpha}{2}\right)^{\frac{t-t_\alpha}{2}} + \alpha t_\alpha\right).\]

Therefore, we can bound $T_1$ as follows,
$$
\begin{aligned}
T_1 & = \frac{1}{\left(k-k_0\right)^2} \sum_{t=k_0}^{k-1} \mathbb{E}\left[\left(q_t-\mathbb{E}\left[q_{\infty}\right]\right)\left(q_t-\mathbb{E}\left[q_{\infty}\right]\right)^{\top}\right] \\
& =\frac{1}{\left(k-k_0\right)^2} \sum_{t=k_0}^{k-1} \mathcal{O}\left(\left(1-\frac{\left(1-\beta\right) \alpha}{2}\right)^{\frac{t-t_\alpha}{2}} + \alpha t_\alpha\right) \\
& =\mathcal{O}\left(\frac{1}{\alpha(k - k_0)^2} \exp\left(-\frac{\alpha(1-\beta)k_0}{4}\right)\right)+\mathcal{O}\left(\frac{\alpha t_\alpha}{k-k_0}\right) \\
& = \mathcal{O}\left(\frac{1}{\alpha(k - k_0)^2} \exp\left(-\frac{\alpha(1-\beta)k_0}{4}\right) + \frac{\alpha t_\alpha}{k-k_0}\right).
\end{aligned}
$$

Regarding the term $T_2$, notice that for $l > t$, we have
$$
\begin{aligned}
\mathbb{E}\left[\left(q_t-\mathbb{E}\left[q_{\infty}\right]\right)\left(q_l-\mathbb{E}\left[q_{\infty}\right]\right)^{\top}\right] & =\mathbb{E}\left[\mathbb{E}\left[\left(q_t-\mathbb{E}\left[q_{\infty}\right]\right)\left(q_l-\mathbb{E}\left[q_{\infty}\right]\right)^{\top} \mid q_t\right]\right] \\
& =\mathbb{E}\left[\left(q_t-\mathbb{E}\left[q_{\infty}\right]\right) \mathbb{E}\left[q_l-\mathbb{E}\left[q_{\infty}\right] \mid q_t\right]^{\top}\right] \\
& =\mathbb{E}\left[\left(q_t-\mathbb{E}\left[q_{\infty}\right]\right)\left(\mathbb{E}\left[q_l \mid q_t\right]-\mathbb{E}\left[q_{\infty}\right]\right)^{\top}\right].
\end{aligned}
$$

Note that for any $y \in \mathbb{R}^d$, it holds that
$$
\left\|\mathbb{E}\left[q_l \mid q_t=y\right]-\mathbb{E}\left[q_{\infty}\right]\right\|=\left\|\mathbb{E}\left[q_{l-t} \mid q_0=y\right]-\mathbb{E}\left[q_{\infty}\right]\right\| \leq C(r,\gamma,P)\cdot \left(1-\frac{\left(1-\beta\right) \alpha}{2}\right)^{\frac{l-t-t_\alpha}{2}},
$$
where the second inequality holds since Corollary \ref{co4limit} holds for all initial value of $q_0$.

Therefore, when $l > t$, we have
$$
\begin{aligned}
& \mathbb{E}\left[\left\|\left(q_t-\mathbb{E}\left[q_{\infty}\right]\right)\left(\mathbb{E}\left[q_l \mid q_t\right]-\mathbb{E}\left[q_{\infty}\right]\right)^{\top}\right\|_\infty\right] \\
\leq & \mathbb{E}\left[\left\|q_t-\mathbb{E}\left[q_{\infty}\right]\right\|_\infty\left\|\mathbb{E}\left[q_l \mid q_t\right]-\mathbb{E}\left[q_{\infty}\right]\right\|_\infty\right] \\
\leq & \mathbb{E}\left[\left\|q_t-\mathbb{E}\left[q_{\infty}\right]\right\|_\infty\right] \cdot\left(C(r,\gamma,P)\cdot \left(1-\frac{\left(1-\beta\right) \alpha}{2}\right)^{\frac{l-t-t_\alpha}{2}}\right) \\
\leq & \left(\mathbb{E}\left[\left\|q_t-q_{\infty}\right\|_\infty\right]+\mathbb{E}\left[\left\|q_{\infty}-\mathbb{E}\left[q_{\infty}\right]\right\|_\infty\right]\right) \cdot\left(C(r,\gamma,P)\cdot \left(1-\frac{\left(1-\beta\right) \alpha}{2}\right)^{\frac{l-t-t_\alpha}{2}}\right) \\
\stackrel{(i)}{\leq} & \left(\mathbb{E}\left[\left\|q_t-q_{\infty}\right\|_\infty\right]+\left(\tr(\var(q_{\infty}))\right)^{1/2}\right) \cdot\left(C(r,\gamma,P)\cdot \left(1-\frac{\left(1-\beta\right) \alpha}{2}\right)^{\frac{l-t-t_\alpha}{2}}\right) \\
\stackrel{(ii)}{\leq} & \left(C(r,\gamma,P)\cdot \left(1-\frac{\left(1-\beta\right) \alpha}{2}\right)^{\frac{t-t_\alpha}{2}}+C'(r,\gamma,P)\sqrt{\alpha t_{\alpha}}\right) \cdot\left(C(r,\gamma,P)\cdot \left(1-\frac{\left(1-\beta\right) \alpha}{2}\right)^{\frac{l-t-t_\alpha}{2}}\right) \\
= & C^2(r,\gamma,P)\cdot \left(1-\frac{\left(1-\beta\right) \alpha}{2}\right)^{\frac{l-2t_\alpha}{2}}+C^{\prime\prime\prime\prime}(r,\gamma,P)\cdot \sqrt{\alpha t_{\alpha}}\cdot\left(1-\frac{\left(1-\beta\right) \alpha}{2}\right)^{\frac{l-t - t_\alpha}{2}},
\end{aligned}
$$
where in $(i)$ $\tr(\cdot)$ denotes the trace operator and we use the fact that $\mathbb{E}\left[\left\|q_{\infty}-\mathbb{E}\left[q_{\infty}\right]\right\|_\infty\right]\leq \sqrt{\mathbb{E}\left[\left\|q_{\infty}-\mathbb{E}\left[q_{\infty}\right]\right\|^2_\infty\right]}=\tr(\var(q_{\infty}))^{1/2}$; in $(ii)$ we use the bounds in equations \eqref{eq:exp_qt} and \eqref{eq:var_qinf}.

In addition, note that
$$
\begin{aligned}
&\quad\frac{1}{\left(k-k_0\right)^2} \sum_{t=k_0}^{k-1} \sum_{l=t+1}^{k-1} \mathcal{O}\left(\left(1-\frac{\left(1-\beta\right) \alpha}{2}\right)^{\frac{l-2t_\alpha}{2}}\right)\\
& \leq \frac{1}{\left(k-k_0\right)^2} \sum_{t=k_0}^{\infty} \sum_{l=t+1}^{\infty} \mathcal{O}\left(\left(1-\frac{\left(1-\beta\right) \alpha}{2}\right)^{\frac{l-2t_\alpha}{2}}\right) \\
& \leq \frac{1}{\left(k-k_0\right)^2}\left(\frac{4}{(1-\beta) \alpha}\right)^2\mathcal{O}\left(\left(1-\frac{\left(1-\beta\right) \alpha}{2}\right)^{\frac{k_0-2t_\alpha}{2}}\right) \\
&= \mathcal{O}\left(\frac{1}{(k - k_0)^2\alpha^2}\exp\left(-\frac{\alpha(1-\beta)k_0}{4}\right)\right),
\end{aligned}
$$
and
$$
\begin{aligned}
&\quad \frac{1}{\left(k-k_0\right)^2}\sum_{t=k_0}^{k-1} \sum_{l=t+1}^{k-1} \mathcal{O}\left(\left(1-\frac{\left(1-\beta\right) \alpha}{2}\right)^{\frac{l-t - t_\alpha}{2}}\right)\\
& \leq \frac{1}{\left(k-k_0\right)^2} \sum_{t=k_0}^{k-1} \sum_{l=t+1}^{\infty} \mathcal{O}\left(\left(1-\frac{\left(1-\beta\right) \alpha}{2}\right)^{\frac{l-t - t_\alpha}{2}}\right) \\
& = \mathcal{O}\left(\frac{1}{(k-k_0)\alpha}\right).
\end{aligned}
$$

Putting together, we obtain the following upper bound for $T_2$,
$$
\begin{aligned}
T_2 & =\frac{1}{\left(k-k_0\right)^2} \sum_{t=k_0}^{k-1} \sum_{l=t+1}^{k-1} \mathcal{O}\left(\left(1-\frac{\left(1-\beta\right) \alpha}{2}\right)^{\frac{l-2t_\alpha}{2}} + \sqrt{\alpha t_{\alpha}}\left(1-\frac{\left(1-\beta\right) \alpha}{2}\right)^{\frac{l-t - t_\alpha}{2}}\right) \\
& =\mathcal{O}\left(\frac{1}{(k - k_0)^2\alpha^2}\exp\left(-\frac{\alpha(1-\beta)k_0}{4}\right) + \frac{\sqrt{\alpha t_{\alpha}}}{(k-k_0)\alpha} \right).
\end{aligned}
$$

Combining the above bounds for $T_1$ and $T_2$, we obtain
\begin{equation}\label{j}
\begin{aligned}
 &\mathbb{E} {\left[\left(\bar{q}_{k_0, k}-\mathbb{E}\left[q_{\infty}\right]\right)\left(\bar{q}_{k_0, k}-\mathbb{E}\left[q_{\infty}\right]\right)^{\top}\right] } \\
= & \mathcal{O}\left(\frac{1}{\alpha(k - k_0)^2} \exp\left(-\frac{\alpha(1-\beta)k_0}{4}\right)+ \frac{\alpha t_\alpha}{k-k_0}\right)  \\
&\quad + \mathcal{O}\left(\frac{1}{(k - k_0)^2\alpha^2}\exp\left(-\frac{\alpha(1-\beta)k_0}{4}\right) + \frac{\sqrt{t_{\alpha}/\alpha}}{(k-k_0)} \right)\\
=& \mathcal{O}\left(  \frac{\sqrt{t_{\alpha}/\alpha}}{(k-k_0)} + \frac{1}{(k - k_0)^2\alpha^2}\exp\left(-\frac{\alpha(1-\beta)k_0}{4}\right)  \right).
\end{aligned}
\end{equation}

Now we are ready to bound the LHS of equation~\eqref{eq:PR_second}. First, we have the following decomposition
\begin{align}
& \mathbb{E}\left[\left(\bar{q}_{k_0, k}-q^*\right)\left(\bar{q}_{k_0, k}-q^*\right)^{\top}\right] \nonumber \\
= & \mathbb{E}\left[\left(\bar{q}_{k_0, k}-\mathbb{E}\left[q_{\infty}\right]+\mathbb{E}\left[q_{\infty}\right]-q^*\right)\left(\bar{q}_{k_0, k}-\mathbb{E}\left[q_{\infty}\right]+\mathbb{E}\left[q_{\infty}\right]-q^*\right)^{\top}\right] \nonumber \\
= & \mathbb{E}\left[\left(\bar{q}_{k_0, k}-\mathbb{E}\left[q_{\infty}\right]\right)\left(\bar{q}_{k_0, k}-\mathbb{E}\left[q_{\infty}\right]\right)^{\top}\right]+\mathbb{E}\left[\left(\mathbb{E}\left[q_{\infty}\right]-q^*\right)\left(\bar{q}_{k_0, k}-\mathbb{E}\left[q_{\infty}\right]\right)^{\top}\right] \nonumber \\
& +\mathbb{E}\left[\left(\bar{q}_{k_0, k}-\mathbb{E}\left[q_{\infty}\right]\right)\left(\mathbb{E}\left[q_{\infty}\right]-q^*\right)^{\top}\right]+\mathbb{E}\left[\left(\mathbb{E}\left[q_{\infty}\right]-q^*\right)\left(\mathbb{E}\left[q_{\infty}\right]-q^*\right)^{\top}\right]. \label{prsecond}
\end{align}  

For the second term of RHS of equation \ref{prsecond}, we have
$$
\begin{aligned}
&\mathbb{E}\left[\left(\bar{q}_{k_0, k}-\mathbb{E}\left[q_{\infty}\right]\right)\left(\mathbb{E}\left[q_{\infty}\right]-q^*\right)^{\top}\right]\\
 =&\frac{1}{k-k_0}\left(\sum_{t=k_0}^{k-1} \mathbb{E}\left[q_t-q_{\infty}\right]\right)\left(\mathbb{E}\left[q_{\infty}\right]-q^*\right)^{\top} \\
= & \mathcal{O}\left(\frac{1}{\alpha(k - k_0)} \exp\left(-\frac{\alpha(1-\beta)k_0}{4}\right)\right)\left(\alpha B(r, \gamma, P) + \mathcal{O}(\alpha^2 + \alpha^2t_{\alpha^2}^2)\right) \\
 =&\mathcal{O}\left(\frac{1}{k - k_0} \exp\left(-\frac{\alpha(1-\beta)k_0}{4}\right)\right).
\end{aligned}
$$

Similarly, we have the same bound for the third term of equation  \eqref{prsecond}. For the last term of RHS of equation \eqref{prsecond}, we have
$$
\begin{aligned}
\mathbb{E}\left[\left(\mathbb{E}\left[q_{\infty}\right]-q^*\right)\left(\mathbb{E}\left[q_{\infty}\right]-q^*\right)^{\top}\right] & =\left(\mathbb{E}\left[q_{\infty}\right]-q^*\right)\left(\mathbb{E}\left[q_{\infty}\right]-q^*\right)^{\top} \\
& =\left(\alpha B(r, \gamma, P) + \mathcal{O}(\alpha^2 + \alpha^2t_{\alpha^2}^2)\right)\left(\alpha B(r, \gamma, P) + \mathcal{O}(\alpha^2 + \alpha^2t_{\alpha^2}^2)\right)^{\top}\\
& =\alpha^2 B^{\prime}(r, \gamma, P)+\mathcal{O}(\alpha^3 + \alpha^3t_{\alpha^2}^2).
\end{aligned}
$$

Combining all these bounds, we obtain
\[
\begin{aligned}
&\quad\mathbb{E}\left[\left(\bar{q}_{k_0, k}-q^*\right)\left(\bar{q}_{k_0, k}-q^*\right)^{\top}\right] \\
&=   \alpha^2 B^{\prime}(r, \gamma, P)+\mathcal{O}(\alpha^3 + \alpha^3t_{\alpha^2}^2)\\
&\quad +\mathcal{O}\left( \frac{\sqrt{t_{\alpha}/\alpha}}{(k-k_0)}  + \frac{1}{(k - k_0)^2\alpha^2}\exp\left(-\frac{\alpha(1-\beta)k_0}{4}\right)  \right)\\
&= \alpha^2 B^{\prime} + \mathcal{O}\left(\alpha^3 + \alpha^3t_{\alpha^2}^2+  \frac{\sqrt{t_{\alpha}/\alpha}}{(k-k_0)} + + \frac{1}{(k - k_0)^2\alpha^2}\exp\left(-\frac{\alpha(1-\beta)k_0}{4}\right)  \right).
\end{aligned}\]
thereby completing the proof of Corollary \ref{co:pr}.

\section{Proof of Corollary \ref{rr}} \label{sec:proof_rr}

In this section, we give the proof of the first and second moment bounds in Corollary \ref{rr}.

\subsection{First Moment}

We have 
$$
\begin{aligned}
\mathbb{E}\left[\widetilde{q}_{k_0, k}^{(\alpha)}\right]-q^*= & \left(2 \bar{q}_{k_0, k}^{(\alpha)}-\bar{q}_{k_0, k}^{(2 \alpha)}\right)-q^* \\
= & 2\left(\bar{q}_{k_0, k}^{(\alpha)}-q^*\right)-\left(\bar{q}_{k_0, k}^{(2 \alpha)}-q^*\right) \\
 \stackrel{(i)}{=} &2  \left(\alpha B(r, \gamma, P) + \mathcal{O}(\alpha^2 + \alpha^2t_{\alpha^2}^2)+\mathcal{O}\left(\frac{1}{\alpha(k - k_0)} \exp\left(-\frac{\alpha(1-\beta)k_0}{4}\right)\right)\right) \\
& \quad-\left(2 \alpha B(r, \gamma, P) + \mathcal{O}(\alpha^2 + \alpha^2t_{\alpha^2}^2)+\mathcal{O}\left(\frac{1}{\alpha(k - k_0)} \exp\left(-\frac{\alpha(1-\beta)k_0}{2}\right)\right)\right) \\
= &\mathcal{O}(\alpha^2 + \alpha^2t_{\alpha^2}^2)+\mathcal{O}\left(\frac{1}{\alpha(k - k_0)} \exp\left(-\frac{\alpha(1-\beta)k_0}{4}\right)\right)
\end{aligned}
$$
where $(i)$ follows from Corollary \ref{co:pr}.

\subsection{Second Moment}
We first introduce the following short-hands:
$$
\begin{aligned}
& u_1:=\bar{q}_{k_0, k}^{(\alpha)}-\mathbb{E}\left[q_{\infty}^{(\alpha)}\right], \quad u_2:=\bar{q}_{k_0, k}^{(2 \alpha)}-\mathbb{E}\left[q_{\infty}^{(2 \alpha)}\right] \\
 \text { and } \quad & v:=2 \mathbb{E}\left[q_{\infty}^{(\alpha)}\right]-\mathbb{E}\left[q_{\infty}^{(2 \alpha)}\right]+q^*.
&
\end{aligned}
$$

With these notations, $\tilde{q}_{k_0, k}-q^*=2 u_1-u_2+v$. We then have the following bound
$$
\begin{aligned}
\left\Vert \mathbb{E}\left[\left(\widetilde{q}_{k_0, k}^{(\alpha)}-q^*\right)\left(\widetilde{q}_{k_0, k}^{(\alpha)}-q^*\right)^{\top}\right]\right\Vert_\infty & \leq \left\Vert \mathbb{E}\left[\left(\widetilde{q}_{k_0, k}^{(\alpha)}-q^*\right)\left(\widetilde{q}_{k_0, k}^{(\alpha)}-q^*\right)^{\top}\right]\right\Vert_2 \\
& =\left\|\mathbb{E}\left[\left(2 u_1-u_2+v\right)\left(2 u_1-u_2+v\right)^{\top}\right]\right\|_2 \\
& \leq \mathbb{E}\left[\left\|2 u_1-u_2+v\right\|_2^2\right] \\
& \leq 3\mathbb{E}\left\|2 u_1\right\|_2^2+3 \mathbb{E}\left\|u_2\right\|_2^2+3\|v\|_2^2 .
\end{aligned}
$$

By equation \eqref{j}, we have
$$
\mathbb{E}\left\|u_1\right\|_2^2=\operatorname{Tr} \big(\mathbb{E}\left[u_1 u_1^{\top}\right]\big)=\mathcal{O}\left(  \frac{\sqrt{t_{\alpha}/\alpha}}{(k-k_0)}+ \frac{1}{(k - k_0)^2\alpha^2}\exp\left(-\frac{\alpha(1-\beta)k_0}{4}\right)  \right).
$$
Similarly, we have
$$
\mathbb{E}\left\|u_2\right\|_2^2=\mathcal{O}\left(   \frac{\sqrt{t_{\alpha}/\alpha}}{(k-k_0)} + \frac{1}{(k - k_0)^2\alpha^2}\exp\left(-\frac{\alpha(1-\beta)k_0}{2}\right)  \right).
$$
By Theorem \ref{thm:bias}, we have $\|v\|_2^2 = \mathcal{O}\left(\alpha^4 + \alpha^{4}t_{\alpha^2}^4\right).$

Combining these bounds together, we have
\begin{align*}
&\quad \mathbb{E}\left[\left(\tilde{q}_{k-k_0}-q^*\right)\left(\tilde{q}_{k-k_0}-q^*\right)^{\top}\right]\\
&=\mathcal{O}\left(\alpha^4 + \alpha^{4}t_{\alpha^2}^4\right) + \mathcal{O}\left(  \frac{\sqrt{t_{\alpha}/\alpha}}{(k-k_0)} + \frac{1}{(k - k_0)^2\alpha^2}\exp\left(-\frac{\alpha(1-\beta)k_0}{4}\right)  \right).
\end{align*}

\section{Experiment Details} \label{sec:appx_experiment}

\paragraph{Tabular case.} Wwe consider two MDPs for our numerical experiments.

The first example is a $1 \times 3$ Gridword with $\mathcal{S} = \{0,1,2\}$ and $\mathcal{A} = \{-1,1\}$. For each step, the agent can walk in two directions: left or right. If the agent walks out of the space, the agent would get a reward of -4 and stay at the same state. Otherwise, the agent can walk to the next state with probability of 0.95 or still stay at the same state with probability of 0.05. For the case that the agent does not exceed the space, the reward function is determined by the current state $r(s,a) = r(s)$ with $r(0) = 0, r(1)=10$ and $r(2) = 0.5$. The discounted factor is set as $\gamma = 0.9$. 

The second example is a classical $4 \times 4$ Gridworld combined with the slippery mechanism in Frozen-Lake. For each step, the agent can walk in four directions: left, up, right or down. Specially, there are two state 
A and B in which the agent can only intend to move to $A^{\prime}$ and $B^{\prime}$. After the action is selected by the behavior policy, the agent will walk in the intended direction with probability of 0.9 else will move in either perpendicular direction with equal probability of 0.05 in both directions. If the agent walks out of the space, the agent would get a reward of -1 and stay in the same state. Otherwise, the reward function is also determined by the current state with $r(A) = 10$, $r(B) = 5$ and $r(s) = 0$ for $s \neq A, B.$ The discounted factor is set as $\gamma = 0.9$.

\paragraph{Linear function approximation.} Our second set of experiments consider Q-learning with linear function approximation. 
More specifically, we consider approximating the Q-function by a linear subspace spanned by basis vectors $\phi=(\phi_{1},\ldots,\phi_{d})^{\top}:\mS \times \mA\to\R^{d}$. The goal is to find $\theta^*$ such that $\tilde{q}_{\theta^*}:=\Phi \theta^*$ best approximates the optimal Q function $q^*$, where $\Phi$ denotes the feature matrix $\Phi=\begin{bmatrix}\phi(s_1,a_1) & \cdots & \phi(s_{|\mS|},a_{|\mA|})\end{bmatrix}^{\top}\in\R^{|\mS||\mA|\times d}.$ 
We assume that $\Phi$ has a full column rank, which is standard in literature~\cite{bertsekas1996neuro,chen2022finite,melo2008analysis}. Note that $\theta^*$ can be calculated by projected value iteration algorithm. 

In this case, the Q-earning algorithm reduces to updating the parameter $\theta \in \R^d$ as follows \cite{bertsekas1996neuro}: 
\begin{align}
    \theta_{k+1} = \theta_k +\alpha \phi(s_k,a_k)\left(r_k + \gamma \max_{a'} \phi(s_{k+1},a')^{\top}\theta_{k}-\phi(s_k,a_k)^{\top}\theta_k\right), \label{eq:q_learning_linear}
\end{align}
where $(s_k,a_k,r_k,s_{k+1})$ is the sample generated by the behavior policy at time step $k.$

For the MDP and feature vectors, we consider a similar setup as the work \cite[Appendix D.1]{chen2022finite}. We provide the detail description here for completeness. We consider an MDP with $|\mS|=20$ states and $|\mA|=5$ actions. We generate the rewards and transition probabilities as follows: for each $(s,a)\in \mS\times \mA$,
\begin{itemize}
    \item The reward $r(s,a)$ is drawn uniformly in $[0,1].$
    \item For the transition probability $T(\cdot|s,a),$ we first obtain $|\mS|$ numbers by uniformly sampling of $[0,1]$, and then normalize these $|\mS|$ numbers by their sum to make it a valid probability distribution.
\end{itemize}

As for the feature matrix, we consider $d=10.$ For each $(s,a),$ each element of $\phi(s,a)$ is drawn from Bernoulli distribution with parameter $p=0.5,$ and then we normalize the features to ensure $\Vert \phi(s,a)\ \Vert\leq 1.$ We repeat this process until the matrix $\Phi$ has a full column rank.

 We set the discounted factor to be $\gamma = 0.5$ and the Markovian data $\{x_k\}_{k \geq 0}$ is generated from a uniformly random behavior policy.

We run Q-learning with linear function approximation \eqref{eq:q_learning_linear} with initialization $\theta_0^{(\alpha)} = \theta^* + 10$  and stepsize $\alpha \in \{0.1,0.2,0.4\}$. We also consider two
diminishing stepsizes: $\alpha_k = 1/\big(1+(1-\gamma)k\big)$ and $\alpha_k = 1/{k^{0.75}}$ as we used in tabular Q-learning.
The simulation results for the Q-learning with linear function approximation are illustrated in Figure~\ref{fig:linear Q}. We plot the $\ell_1$-norm error $\|\bar{\theta}_{k/2,k}^{(\alpha)} - \theta^*\|_1$ for the tail-averaged (TA) iterates $\bar{\theta}_{k/2,k}^{(\alpha)}$,  the RR extrapolated iterates $\widetilde{\theta}_{k}^{(\alpha)}$ with stepsizes $\alpha$ and $2\alpha$, and iterates with diminishing stepsizes.

\begin{figure}[ht]
\centering
\includegraphics[width=0.65\columnwidth]{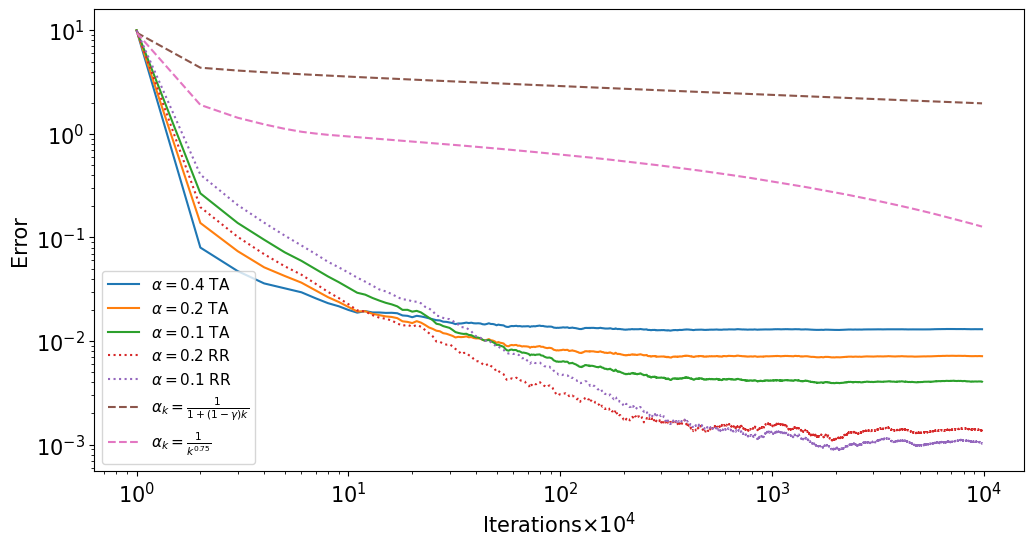} 
\caption{The Q-learning with linear function approximation errors of tail-averaged (TA) iterates and RR extrapolated iterates with different stepsizes.}
\label{fig:linear Q}
\end{figure}

We can observe some similar results as tabular Q-learning's:

\begin{itemize}
    \item  The larger the stepsize $\alpha,$ the faster it converges.
    \item  The final TA error, which corresponds to the asymptotic bias, is approximately proportional to the stepsize.
    \item  RR extrapolated iterates reduce the bias.
    \item  The TA and RR-extrapolated iterates with constant stepsizes enjoy significantly faster initial convergence than those with diminishing stepsizes.
\end{itemize}

\end{document}